\documentclass{article}

\PassOptionsToPackage{numbers, compress}{natbib}

\usepackage[preprint]{neurips_2026}

\usepackage{amsmath,amssymb,amsthm,mathtools}
\usepackage[utf8]{inputenc}
\usepackage[T1]{fontenc}
\usepackage{hyperref}
\usepackage{url}
\usepackage{booktabs}
\usepackage{amsfonts}
\usepackage{nicefrac}
\usepackage{microtype}
\usepackage{xcolor}
\usepackage{bm}
\usepackage{graphicx}
\usepackage{algorithm}
\usepackage{algpseudocode}
\usepackage{enumitem}
\newcommand{\indic}{\mathbf{1}}
\usepackage{wrapfig}
\usepackage{subcaption}
\usepackage{multirow}

\newtheorem{theorem}{Theorem}
\newtheorem{proposition}{Proposition}

\newtheorem{definition}{Definition}

\newtheorem{corollary}{Corollary}

\newcommand{\E}{\mathbb{E}}
\newcommand{\R}{\mathbb{R}}
\newcommand{\Prb}{\mathbb{P}}
\newcommand{\cD}{\mathcal{D}}
\newcommand{\cB}{\mathcal{B}}
\newcommand{\cX}{\mathcal{X}}

\newcommand{\cL}{\mathcal{L}}

\newcommand{\Acc}{\mathrm{Acc}}
\newcommand{\Cost}{\mathrm{C}}
\newcommand{\Reg}{\mathrm{Reg}}
\DeclareMathOperator*{\argmax}{arg\,max}

\definecolor{cblue}{HTML}{3366CC}
\definecolor{cred}{HTML}{CC3333}
\definecolor{cgray}{HTML}{666666}

\title{Adaptive Test-Time Compute Allocation for Reasoning LLMs via Constrained Policy Optimization}

\author{%
  Zhiyuan Zhai\thanks{Correspondence to: \texttt{22110720067@m.fudan.edu.cn}} \\
  Fudan University \\
  \texttt{22110720067@m.fudan.edu.cn}
  \And
  Bingcong Li \\
  ETH Zurich \\
  \texttt{bingtsongli@gmail.com}
  \And
  Bingnan Xiao \\
  Fudan University \\
  \texttt{xbn20000224@gmail.com}
  \And
  Ming Li \\
  Guangming Lab \\
  \texttt{ming.li@u.nus.edu}
  \And
  Xin Wang \\
  Fudan University \\
  \texttt{xwang11@fudan.edu.cn}
}

\begin{document}

\maketitle

\begin{abstract}
Test-time compute scaling, the practice of spending extra computation during inference via repeated sampling, search, or extended reasoning, has become a powerful lever for improving large language model performance.
Yet deploying these techniques under finite inference budgets requires a decision that current systems largely ignore: \emph{which inputs deserve more compute, and which can be answered cheaply?}
We formalize this as a constrained optimization problem (maximize expected accuracy subject to an average compute budget) and solve it with a two-stage \textsc{Solve-then-Learn} pipeline.
In the \emph{solve} stage, Lagrangian relaxation decomposes the global constraint into per-instance sub-problems, each admitting a closed-form oracle action that optimally prices accuracy against cost.
We prove that the induced cost is monotone in the dual variable, enabling exact budget targeting via binary search.
In the \emph{learn} stage, a lightweight classifier is trained to predict oracle actions from cheap input features, amortizing the allocation rule for real-time deployment.
We establish that the task-level regret of the learned policy is bounded by its imitation error times the worst-case per-instance gap, yielding a clean reduction from constrained inference to supervised classification.
Experiments on MATH and GSM8K with three LLMs (DeepSeek-V3, GPT-4o-mini, Qwen2.5-7B) show that our method consistently outperforms uniform and heuristic allocation baselines, achieving up to 12.8\% relative accuracy improvement on MATH under matched budget constraints, while closely tracking the Lagrangian oracle upper bound with over 91\% imitation accuracy.
\end{abstract}

\noindent Code is available at \url{https://github.com/zhiyuanZhai20/AdaCompute-LLM}.

\section{Introduction}\label{sec:intro}

The dominant recipe for improving large language model (LLM) performance has long been to increase pre-training compute: more parameters, more data, more FLOPs~\cite{kaplan2020scaling, hoffmann2022training}.
A striking recent development is that comparable, and sometimes larger, gains can be achieved at \emph{inference time}, by spending extra computation on each query through techniques such as chain-of-thought reasoning~\cite{wei2022chain}, self-consistency voting~\cite{wang2023selfconsistency}, tree search~\cite{yao2024tree}, or process-reward verification~\cite{lightman2024lets}.
This has been termed \emph{test-time scaling}, and it underpins the capabilities of frontier reasoning systems~\cite{openai2024o1, deepseek2025r1}.

These inference-time techniques, however, are expensive: generating $N$ chain-of-thought responses for best-of-$N$ or majority voting multiplies cost linearly, yet many questions are answerable correctly in a single attempt.
In practice, inference budgets are finite.
API costs, GPU availability, latency targets, and throughput requirements impose hard or soft limits on total inference compute, and as agentic AI systems adopt ever-longer multi-step reasoning pipelines~\cite{openai2024o1}, the gap between per-query cost and per-query necessity will only widen.
The standard approach allocates a \emph{uniform} budget to every input, applying the same number of samples regardless of difficulty~\cite{wang2023selfconsistency, brown2024large}.
This is wasteful: easy inputs that a single forward pass can solve still consume the full budget, while hard inputs that would benefit from more compute are starved at the same level.

The heterogeneity of real workloads creates an opportunity: by redistributing compute from easy inputs to hard ones, the same total budget can yield significantly higher aggregate accuracy.
But which inputs should receive more compute?
Figure~\ref{fig:motivation} illustrates the challenge concretely.
We define the \emph{budget} $b$ as the number of independent responses sampled from the LLM, with the final answer determined by majority vote (self-consistency); the cost function $\Cost(b)$ maps each budget level to its resource consumption (e.g., $\Cost(b) = b$ when cost is proportional to the number of samples).
Some questions are trivially correct at $b=1$ (a single response), making additional samples purely wasteful; others exhibit steep accuracy gains as $b$ increases (``responsive''); still others show diminishing returns or remain intractable regardless of budget.
In hindsight, the optimal strategy is clear: assign $b=1$ to easy questions, concentrate high budgets on responsive ones, and avoid wasting resources on intractable ones.
The challenge is to approximate this oracle strategy \emph{before} observing the outcomes.
Realizing this opportunity requires answering a precise question:

\begin{quote}
\emph{Given a finite average compute budget shared across a batch of inputs, how should one allocate inference compute to individual inputs to maximize aggregate performance?}
\end{quote}

We observe that this question has the structure of a \emph{constrained optimization problem}:
\begin{equation}\label{eq:main}
  \max_{\pi}\;\E_{x\sim\cD}\!\big[\Acc(x,\pi(x))\big]
  \qquad
  \text{s.t.}\quad
  \E_{x\sim\cD}\!\big[\Cost(\pi(x))\big] \le B,
\end{equation}

\begin{wrapfigure}{r}{0.43\textwidth}
\vspace{-12pt}
\centering
\includegraphics[width=0.42\textwidth]{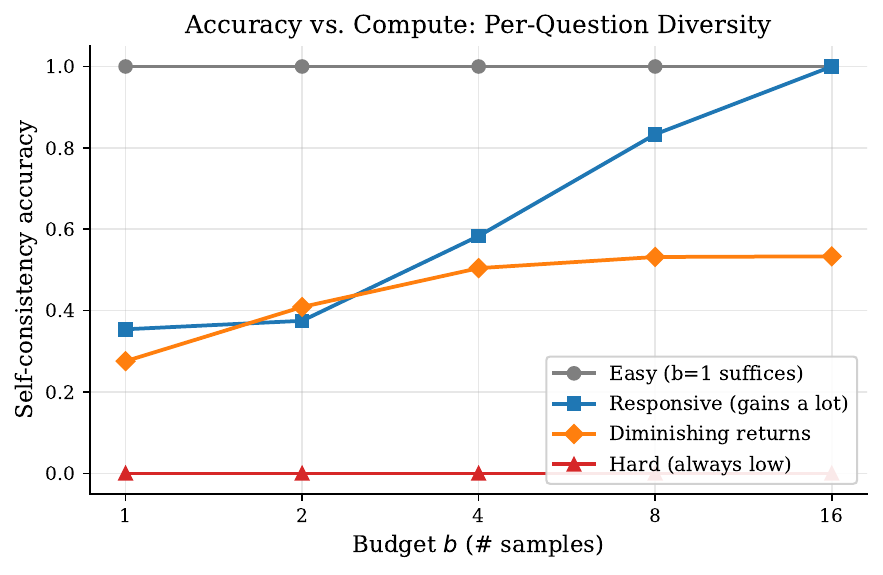}
\caption{Self-consistency accuracy vs.\ budget $b$ (number of sampled responses used for majority voting) for four representative MATH questions (DeepSeek-V3). Different questions exhibit dramatically different scaling behaviour, motivating input-adaptive allocation.}
\label{fig:motivation}
\vspace{-10pt}
\end{wrapfigure}

\noindent where $\pi$ is a policy mapping inputs to compute budgets, $\Acc(x,b)$ is the accuracy achievable on input $x$ with budget $b$, $\Cost(\cdot)$ is a cost function mapping budgets to resource consumption (formally defined in \S\ref{sec:formulation}), and $B$ is the target average cost.
The constraint couples all inputs through a shared resource, so the globally optimal solution cannot be found by solving each input independently.

We propose a two-stage \textsc{Solve-then-Learn} framework.
In the \textbf{solve stage} (\S\ref{sec:oracle}), we attach a Lagrange multiplier $\lambda$ to the budget constraint.
The resulting relaxation decomposes across inputs: each input's sub-problem has a closed-form oracle solution $b^\star(x;\lambda) = \argmax_{b}[\Acc(x,b) - \lambda\,\Cost(b)]$, which optimally balances the marginal accuracy gain against a ``shadow price'' $\lambda$ for compute.
We prove that the aggregate cost under the oracle is monotone in $\lambda$, enabling binary search to match any target budget $B$.
In the \textbf{learn stage} (\S\ref{sec:amortize}), we train a lightweight policy network to predict oracle labels from cheap input features, amortizing the allocation rule for real-time deployment.

This decomposition yields a theoretically clean framework with formal guarantees.
We summarize the contributions of this paper as follows.

\paragraph{Contributions.}
\begin{itemize}[leftmargin=1.5em,itemsep=2pt]
\item We are the first to formalize input-adaptive test-time compute allocation as a constrained optimization problem (\S\ref{sec:formulation}), providing a principled alternative to heuristic rules and enabling systematic study of the accuracy--cost tradeoff under finite inference budgets.
\item We propose a \textsc{Solve-then-Learn} framework (\S\ref{sec:oracle}--\S\ref{sec:amortize}) that decomposes the global constrained problem into per-instance sub-problems via Lagrangian relaxation, then amortizes the oracle allocation into a lightweight classifier for real-time deployment.
\item We establish formal guarantees (\S\ref{sec:theory}): cost monotonicity enabling binary search for budget targeting, an imitation regret bound reducing constrained inference to supervised classification, and a primal recovery result ensuring near-feasibility and near-optimality.
\end{itemize}


\section{Problem Formulation: Budget-Constrained Compute Allocation}\label{sec:formulation}

\noindent\textbf{Setup.}\;
Let $\cX$ denote the input space and $\cD$ the deployment distribution over $\cX$.
We have a discrete set of compute budgets $\cB = \{b_1, \ldots, b_K\}$, ordered by cost: $\Cost(b_1) \le \cdots \le \Cost(b_K)$.
Each budget $b_k$ represents a specific operating mode of the inference engine---e.g., generating $k$ candidate responses and taking the best (best-of-$K$), running self-consistency with $k$ samples, or allowing a chain-of-thought of length proportional to $k$.
For each input $x \in \cX$ and budget $b \in \cB$, we write $\Acc(x,b) \in [0,1]$ for the expected task utility (e.g., probability of producing a correct answer).
A \emph{budget allocation policy} is a measurable function $\pi: \cX \to \cB$ that assigns a compute budget to each input, inducing two aggregate quantities: the expected accuracy $\E[\Acc(x,\pi(x))]$ and the expected cost $\E[\Cost(\pi(x))]$.

Given a target average budget $B > 0$, the design problem is the constrained program
\begin{equation}\label{eq:primal}
  V(B) \;:=\;
  \max_{\pi \in \Pi}\;
  \E_{x\sim\cD}\!\big[\Acc(x,\pi(x))\big]
  \qquad
  \text{s.t.}\quad
  \E_{x\sim\cD}\!\big[\Cost(\pi(x))\big] \le B,
\end{equation}
where $\Pi$ is the class of all measurable maps from $\cX$ to $\cB$ and $V(B)$ denotes the optimal value.
On a finite dataset $\{x_i\}_{i=1}^N$, we work with the empirical version
\begin{equation}\label{eq:emp_primal}
  \hat{V}(B) \;:=\;
  \max_{\pi}\;
  \frac{1}{N}\sum_{i=1}^{N}\Acc\!\big(x_i,\pi(x_i)\big)
  \qquad
  \text{s.t.}\quad
  \frac{1}{N}\sum_{i=1}^{N}\Cost\!\big(\pi(x_i)\big) \le B.
\end{equation}

\noindent\textbf{Why is this hard?}\;
Three features distinguish~\eqref{eq:primal} from textbook optimization.
\emph{(i)~Black-box reward}: $\Acc(x,b)$ has no analytic form and can only be estimated by running the model.
\emph{(ii)~Functional decision variable}: the optimizer is a policy function, not a scalar or vector.
\emph{(iii)~Global coupling}: the budget constraint links all inputs through a shared resource, so independently maximizing per-input accuracy is infeasible whenever $B < \Cost(b_K)$.

\section{The Lagrangian Oracle}\label{sec:oracle}

We now show how Lagrangian duality converts the globally coupled problem~\eqref{eq:primal} into a collection of independent per-instance sub-problems.

\subsection{Dual Relaxation}

Introducing a Lagrange multiplier $\lambda \ge 0$ for the budget constraint, the Lagrangian of~\eqref{eq:emp_primal} is
\begin{equation}\label{eq:lagrangian}
  \cL(\pi, \lambda)
  \;=\;
  \frac{1}{N}\sum_{i=1}^{N}
  \underbrace{\Big[\Acc\!\big(x_i,\pi(x_i)\big)
  - \lambda\,\Cost\!\big(\pi(x_i)\big)\Big]}_{f_\lambda(x_i,\,\pi(x_i))}
  \;+\; \lambda\,B.
\end{equation}
The crucial observation is that the additive structure of~\eqref{eq:lagrangian} \emph{decouples} across inputs: for fixed $\lambda$, maximizing $\cL(\pi,\lambda)$ over $\pi$ reduces to independently maximizing $f_\lambda(x_i, b)$ over $b$ for each input $x_i$.

\begin{definition}[Oracle allocation]\label{def:oracle}
For dual variable $\lambda \ge 0$, the \emph{oracle allocation} for input $x$ is
\begin{equation}\label{eq:oracle_label}
  b^\star(x;\lambda) \;:=\; \argmax_{b \in \cB}\;\big[\Acc(x,b) - \lambda\,\Cost(b)\big].
\end{equation}
\end{definition}

The oracle allocation has a transparent economic interpretation.
The quantity $f_\lambda(x,b) = \Acc(x,b) - \lambda\,\Cost(b)$ is the \emph{net value} of assigning budget $b$ to input $x$, where the dual variable $\lambda$ acts as a \emph{unit price of compute}.
The oracle selects the budget that maximizes net value: it spends more on inputs where the marginal accuracy gain justifies the cost, and economizes on inputs where it does not.
When $\lambda = 0$, compute is free and the oracle assigns the maximum budget to every input.
As $\lambda$ increases, the price of compute rises and the oracle progressively shifts to cheaper budgets, starting with the easiest inputs where additional compute yields the smallest accuracy gain.

\subsection{Budget Targeting via Cost Monotonicity}

The oracle allocation induces an average cost that depends on $\lambda$:
\begin{equation}\label{eq:avg_cost}
  \bar{C}(\lambda)
  \;:=\;
  \frac{1}{N}\sum_{i=1}^{N} \Cost\!\big(b^\star(x_i;\lambda)\big).
\end{equation}
The following result shows that $\bar{C}(\lambda)$ is well-behaved as a function of $\lambda$.

\begin{theorem}[Cost monotonicity]\label{thm:mono}
For finite $\cB$, there exists a deterministic tie-breaking rule under which $\bar{C}(\lambda)$ is non-increasing in $\lambda$.
\end{theorem}

\noindent The proof (in Appendix~\ref{app:proofs}) follows from observing that each pairwise net-value difference $f_\lambda(x,b) - f_\lambda(x,b')$ is affine and decreasing in $\lambda$ whenever $\Cost(b) > \Cost(b')$, so each input's oracle budget can only shift cheaper as $\lambda$ grows.

\paragraph{Practical implication.}

\begin{wrapfigure}{r}{0.46\textwidth}
\vspace{-14pt}
\centering
\includegraphics[width=0.45\textwidth]{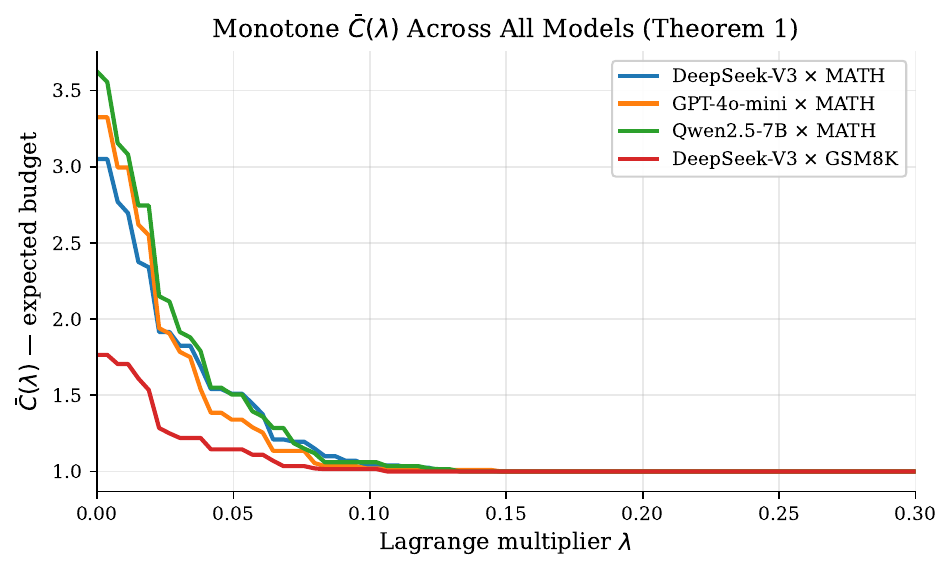}
\caption{Empirical validation of Theorem~\ref{thm:mono}: $\bar{C}(\lambda)$, the average cost under the oracle allocation (Eq.~\ref{eq:avg_cost}), is monotonically non-increasing for all four model--dataset combinations.}
\label{fig:monotone}
\vspace{-8pt}
\end{wrapfigure}

Theorem~\ref{thm:mono} implies that for any target budget $B \in [\Cost(b_1), \Cost(b_K)]$, one can find a dual variable $\lambda_B$ satisfying $\bar{C}(\lambda_B) \approx B$ via binary search.
Since evaluating $\bar{C}(\lambda)$ requires only an $\argmax$ over $K$ options per input, each iteration costs $O(NK)$ and the search converges in $O(\log(1/\varepsilon))$ iterations.
We denote the resulting oracle labels as $y_i^\star := b^\star(x_i;\lambda_B)$.
Figure~\ref{fig:monotone} confirms this property empirically across all four model--dataset combinations: $\bar{C}(\lambda)$ traces a monotonically non-increasing staircase, with the discrete steps reflecting the finite budget set $\cB=\{1,2,4,8,16\}$.

\subsection{Dual Bound and Optimality}\label{sec:dual-bound}

The dual function is $G(\lambda):=\max_\pi\cL(\pi,\lambda)$.
Weak duality gives $\hat{V}(B)\le G(\lambda)$ for all $\lambda\ge 0$.
The tightest upper bound is $G(\lambda_B)$, and the gap $G(\lambda_B)-\hat{V}(B)$ is the \emph{duality gap}.
For the empirical problem~(3), which is a finite program over $N$ independent discrete variables with a single coupling constraint, strong duality holds~\cite{boyd2004convex, bertsekas2015convex}:

\begin{proposition}[Strong duality]\label{prop:strong}
The empirical problem~(3) has zero duality gap: $\hat{V}(B)=\min_{\lambda\ge 0}G(\lambda)$.
\end{proposition}

\noindent
Since the duality gap is zero, the oracle allocation $\{b^\star(x_i;\lambda_B)\}$ obtained via binary search is an \emph{exact} solution to the original constrained problem~(3).
A complete proof---including the LP relaxation construction, weak duality, the subgradient characterisation connecting budget targeting to dual minimisation, and the vertex argument establishing strong duality---is given in Appendix~\ref{app:stochastic}.

\section{Amortized Policy Learning}\label{sec:amortize}

\begin{figure}[t]
\centering
\includegraphics[width=0.95\textwidth]{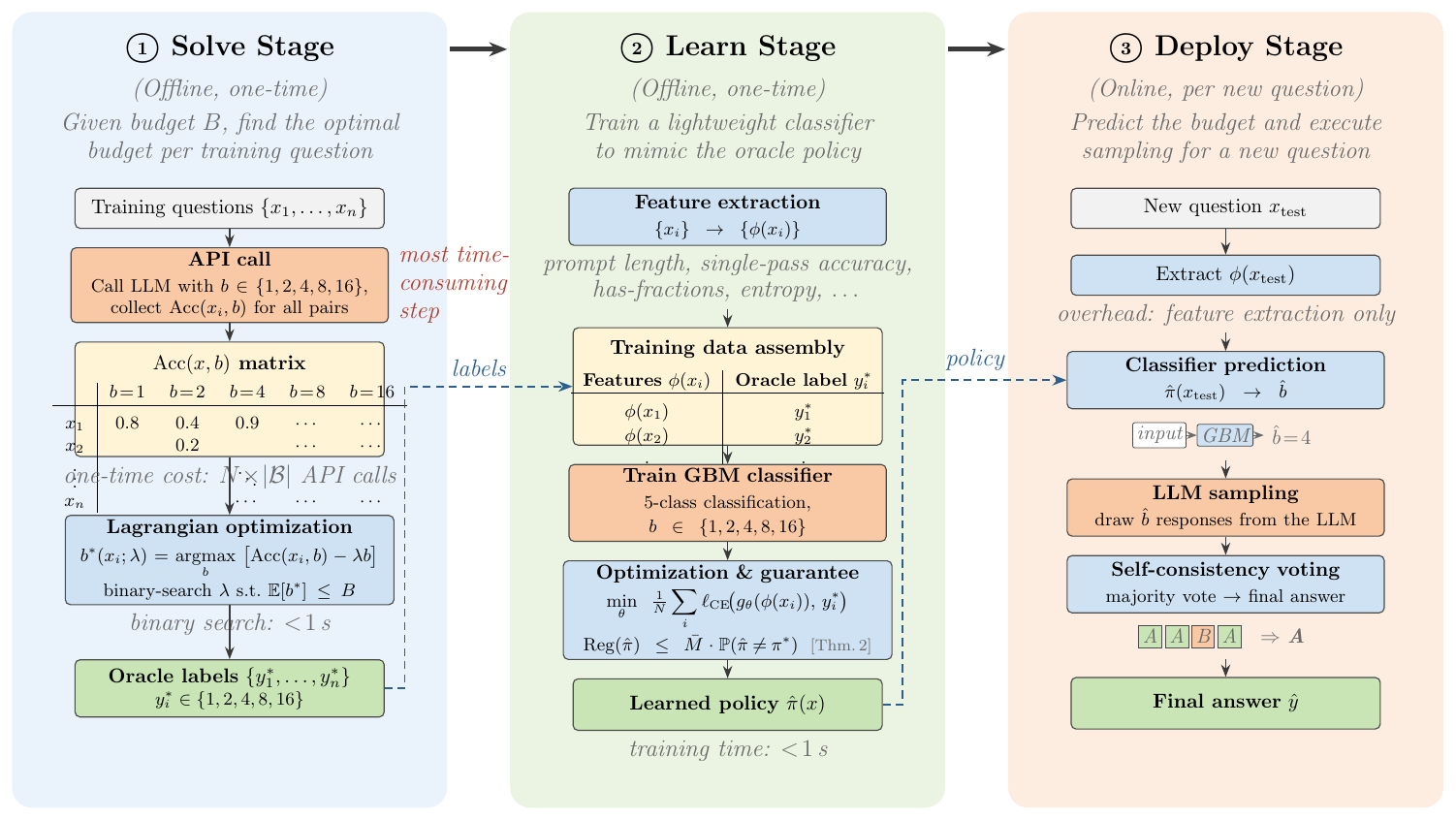}
\caption{Overview of the \textsc{Solve-then-Learn} pipeline. \textbf{Solve stage} (offline): compute the $\Acc(x,b)$ matrix via API calls, then derive oracle labels via Lagrangian optimisation with binary search over $\lambda$. \textbf{Learn stage} (offline): train a lightweight GBM classifier on oracle labels using cheap text features. \textbf{Deploy stage} (online): predict budget for new questions at inference time. The entire training overhead is $<$1s; only the one-time API data collection is expensive.}
\label{fig:pipeline}
\end{figure}

The oracle labels $\{y_i^\star\}$ define an optimal budget assignment on the training set, but computing them for a new input $x$ requires access to $\Acc(x,b)$ for all $b \in \cB$---which is precisely the expensive inference we aim to avoid.
The second stage of our pipeline (Figure~\ref{fig:pipeline}) learns to \emph{predict} oracle labels from cheap input features, amortizing the allocation rule for deployment.

\subsection{Feature Extraction}

Let $\phi(x) \in \R^d$ denote a feature vector computed \emph{before} the expensive inference call.
The framework is agnostic to the choice of features; the only requirement is that $\phi(x)$ be inexpensive relative to full inference and informative about the difficulty of $x$.
Candidates include lexical statistics (input length, vocabulary diversity, number of sub-questions), embedding-based features (prompt embeddings from a frozen encoder), and cheap inference signals (confidence or entropy from a single draft generation).

\subsection{Policy as Classification}

Since $\cB$ is discrete with $|\cB| = K$, the policy is a $K$-way classifier.
Specifically, $\pi_\theta$ parameterizes a function $g_\theta: \R^d \to \R^K$ and predicts
\begin{equation}\label{eq:policy_pred}
  \pi_\theta(x) = b_{\argmax_k\, g_\theta(\phi(x))_k}.
\end{equation}
Training minimizes the cross-entropy loss against oracle labels:
\begin{equation}\label{eq:ce_loss}
  \min_\theta\;
  \frac{1}{N}\sum_{i=1}^{N}
  \ell_{\mathrm{CE}}\!\big(g_\theta(\phi(x_i)),\; y_i^\star\big).
\end{equation}
The classifier can be any standard architecture---logistic regression, gradient-boosted trees, or a small MLP---and trains in seconds to minutes on the pre-computed features.

\subsection{Deployment}

At inference time, a new input $x$ is processed in two steps: (1)~\textbf{Route}: extract features $\phi(x)$ and run $\pi_\theta$ to select budget $\hat{b} = \pi_\theta(x)$; (2)~\textbf{Execute}: run the inference engine once at budget $\hat{b}$.
Since $\pi_\theta$ is a lightweight classifier, routing adds negligible overhead ($<$1\% of the inference cost), making adaptive allocation practical even under strict latency constraints.

\begin{algorithm}[t]
\caption{\textsc{Solve-then-Learn}: Oracle-Supervised Adaptive Budget Allocation}
\label{alg:main}
\begin{algorithmic}[1]
\Require Training inputs $\{x_i\}_{i=1}^N$, budget set $\cB=\{b_1,\ldots,b_K\}$, target budget $B$, feature map $\phi$
\Statex
\Statex \textbf{Stage 1: Solve} \textit{(offline, parallelizable)}
\For{each input $x_i$ and budget $b_k \in \cB$}
    \State Run inference engine; record utility $\Acc(x_i, b_k)$
\EndFor
\State Binary-search for $\lambda_B$ such that $\bar{C}(\lambda_B) \approx B$ \Comment{Theorem~\ref{thm:mono}}
\State Assign oracle labels: $y_i^\star \leftarrow \argmax_{b \in \cB}\big[\Acc(x_i,b) - \lambda_B\,\Cost(b)\big]$ for all $i$
\Statex
\Statex \textbf{Stage 2: Learn} \textit{(fast supervised training)}
\State Extract cheap features $\{\phi(x_i)\}_{i=1}^N$
\State Train classifier $\pi_\theta$ on $\{(\phi(x_i), y_i^\star)\}_{i=1}^N$ via cross-entropy \Comment{Eq.~\eqref{eq:ce_loss}}
\Statex
\Statex \textbf{Deployment} \textit{(per input)}
\State \Return $\pi_\theta(x) = b_{\argmax_k\, g_\theta(\phi(x))_k}$; execute inference at selected budget
\end{algorithmic}
\end{algorithm}

\noindent\textbf{Computational cost.}\;
Stage~1 requires $N \times K$ model evaluations---the dominant cost---but is embarrassingly parallel and performed entirely offline.
The binary search over $\lambda$ operates on the pre-computed utility table and costs $O(NK \log(1/\varepsilon))$.
Stage~2 trains a small classifier and finishes in seconds.
At deployment, only one lightweight classification plus one model call per input is needed.

\subsection{Theoretical Guarantees}\label{sec:theory}

We establish formal guarantees connecting oracle imitation to task-level performance.
All proofs are deferred to Appendix~\ref{app:proofs}.
Define the \emph{penalized value} $J_\lambda(\pi) := \E[\Acc(x,\pi(x)) - \lambda\,\Cost(\pi(x))]$ and the \emph{relaxed regret} $\Reg_\lambda(\hat{\pi}) := J_\lambda(\pi^\star_\lambda) - J_\lambda(\hat{\pi})$.

\begin{theorem}[Imitation regret bound]\label{thm:regret}
Let $\hat{\pi}$ be any policy, $\pi^\star_\lambda$ the oracle at dual variable $\lambda$, and $M_\lambda(x) := \max_{b \in \cB}\, f_\lambda(x, \pi^\star_\lambda(x)) - f_\lambda(x, b)$ the worst-case per-instance gap. Then
\begin{equation}\label{eq:regret_bound_detailed}
  \Reg_\lambda(\hat{\pi})
  \;\le\;
  \E\!\Big[M_\lambda(x) \cdot \indic\!\big\{\hat{\pi}(x) \ne \pi^\star_\lambda(x)\big\}\Big]
  \;\le\;
  \bar{M} \cdot \Prb\!\big(\hat{\pi}(x) \ne \pi^\star_\lambda(x)\big),
\end{equation}
where $\bar{M} := \sup_x M_\lambda(x)$.
\end{theorem}

\noindent Theorem~\ref{thm:regret} provides a clean reduction: \emph{every percentage point of imitation accuracy translates into a proportional decrease in task-level regret}.
The relaxed regret further decomposes into an accuracy gap and a cost deviation (Corollary~\ref{cor:decomp} in Appendix~\ref{app:proofs}), confirming that when the learned policy's average cost is close to the oracle's, the accuracy gap is controlled entirely by the imitation error.

\noindent The analysis above concerns the relaxed problem.
The following result connects back to the original constrained problem~\eqref{eq:primal}.

\begin{theorem}[Primal recovery]\label{thm:recovery}
Let $\lambda_B$ be chosen so that $\E[\Cost(\pi^\star_{\lambda_B}(x))] \le B$, and let $\varepsilon := \Prb(\hat{\pi}(x) \ne \pi^\star_{\lambda_B}(x))$. Then: (1)~\textbf{Near-feasibility}: $\E[\Cost(\hat{\pi}(x))] \le B + \varepsilon \cdot (\Cost(b_K) - \Cost(b_1))$; (2)~\textbf{Near-optimality}: $V(B') - \E[\Acc(x,\hat{\pi}(x))] \le \bar{M} \cdot \varepsilon$, where $B' = B + \varepsilon \cdot (\Cost(b_K) - \Cost(b_1))$.
\end{theorem}

\noindent As the classifier improves ($\varepsilon \to 0$), both the feasibility violation and the optimality gap vanish.

\section{Experiments}\label{sec:experiments}

We evaluate the \textsc{Solve-then-Learn} framework on mathematical reasoning benchmarks, demonstrating consistent gains over baselines across multiple LLMs and budget levels.

\subsection{Experimental Setup}\label{sec:exp_setup}

\noindent\textbf{Benchmarks.}\;
We use two datasets: \textbf{MATH}~\cite{hendrycks2021math}, a competition-level mathematics benchmark (200 questions stratified across difficulty levels~2--4), and \textbf{GSM8K}~\cite{cobbe2021gsm8k}, a grade-school word problem set (200 questions) serving as an ``easy benchmark'' control where the base model already achieves $>$94\% accuracy.

\noindent\textbf{Backbones.}\;
We test three LLMs to demonstrate generality: \textbf{DeepSeek-V3}~\cite{deepseek2024v3} (\texttt{deepseek-chat}), a strong open-weight reasoning model; \textbf{GPT-4o-mini}~\cite{achiam2023gpt4}, OpenAI's compact model; and \textbf{Qwen2.5-7B}~\cite{yang2024qwen25}, Alibaba's 7B-parameter model. All models are queried with temperature $T=0.7$; we collect 48 responses per question to evaluate all budget levels with non-overlapping windows.
The budget set is $\cB = \{1, 2, 4, 8, 16\}$. For each question $x_i$ and budget $b$, we estimate self-consistency accuracy $\Acc(x_i, b)$ by partitioning 48 responses into $\lfloor 48/b \rfloor$ non-overlapping windows of size $b$, taking majority votes within each window, and averaging.

\noindent\textbf{Policy and baselines.}\;
The classifier is a gradient-boosted machine (GBM, XGBoost~\cite{chen2016xgboost}) trained on 16 lightweight text features: lexical statistics (length, word count, sentence count), structural indicators (question marks, numbers, fractions, multi-step keywords), and a normalised entropy estimate from a single-pass LLM call. Training uses an 80/20 split with 3 random seeds; all results report mean $\pm$ std. Full feature and hyperparameter details are in Appendix~\ref{app:setup_details}.
We compare against: (1)~\textbf{Oracle} (Lagrangian-optimal allocation; upper bound requiring knowledge of $\Acc(x,b)$), (2)~\textbf{Fixed-$b$} (uniform budget: largest $b \in \cB$ with $b \le B$), (3)~\textbf{Random} (random budget matching $\E[b]=B$), and (4)~\textbf{Heuristic} (rank by prompt length as difficulty proxy, assign high budgets to longest prompts).

\subsection{Main Results}\label{sec:main_results}

Table~\ref{tab:main} presents task accuracy and realized cost for all methods across four model--dataset pairs and three budget constraints $B \in \{1.5, 2.0, 3.0\}$.

\begin{table*}[t]
\centering
\small
\caption{Main results: task accuracy and realized cost under budget constraints $B\in\{1.5, 2.0, 3.0\}$, mean$\pm$std over 3 seeds. \textbf{Bold}: best non-oracle result per column.}
\label{tab:main}
\setlength{\tabcolsep}{2.8pt}
\resizebox{\textwidth}{!}{%
\begin{tabular}{lcccccccccccc}
\toprule
& \multicolumn{6}{c}{DeepSeek-V3 $\times$ MATH} & \multicolumn{6}{c}{GPT-4o-mini $\times$ MATH} \\
\cmidrule(lr){2-7} \cmidrule(lr){8-13}
Method & $B\!=\!1.5$ & Cost & $B\!=\!2.0$ & Cost & $B\!=\!3.0$ & Cost & $B\!=\!1.5$ & Cost & $B\!=\!2.0$ & Cost & $B\!=\!3.0$ & Cost \\
\midrule
Fixed-$b$ & 0.518 & 1.00 & 0.517 & 2.00 & 0.517 & 2.00 & 0.474 & 1.00 & 0.474 & 2.00 & 0.474 & 2.00 \\
Random & 0.519 & 1.49 & 0.517 & 2.00 & 0.534 & 2.98 & 0.475 & 1.49 & 0.474 & 2.00 & 0.487 & 2.98 \\
Heuristic & 0.521 & 1.52 & 0.525 & 2.05 & 0.529 & 3.02 & 0.479 & 1.52 & 0.480 & 2.05 & 0.485 & 3.02 \\
\midrule
\textbf{Ours} & \textbf{0.550{\tiny$\pm$.001}} & 1.45 & \textbf{0.562{\tiny$\pm$.000}} & 1.87 & \textbf{0.575{\tiny$\pm$.001}} & 2.73 & \textbf{0.501{\tiny$\pm$.001}} & 1.47 & \textbf{0.512{\tiny$\pm$.001}} & 1.87 & \textbf{0.526{\tiny$\pm$.001}} & 2.85 \\
Oracle & 0.555 & 1.51 & 0.571 & 1.95 & 0.586 & 3.02 & 0.507 & 1.54 & 0.520 & 1.98 & 0.538 & 3.02 \\
\midrule
& \multicolumn{6}{c}{Qwen2.5-7B $\times$ MATH} & \multicolumn{6}{c}{DeepSeek-V3 $\times$ GSM8K} \\
\cmidrule(lr){2-7} \cmidrule(lr){8-13}
Method & $B\!=\!1.5$ & Cost & $B\!=\!2.0$ & Cost & $B\!=\!3.0$ & Cost & $B\!=\!1.5$ & Cost & $B\!=\!2.0$ & Cost & $B\!=\!3.0$ & Cost \\
\midrule
Fixed-$b$ & 0.495 & 1.00 & 0.491 & 2.00 & 0.491 & 2.00 & 0.946 & 1.00 & 0.946 & 2.00 & 0.946 & 2.00 \\
Random & 0.493 & 1.49 & 0.491 & 2.00 & 0.515 & 2.98 & 0.946 & 1.49 & 0.946 & 2.00 & 0.952 & 2.98 \\
Heuristic & 0.500 & 1.52 & 0.500 & 2.05 & 0.503 & 3.02 & 0.955 & 1.52 & 0.958 & 2.05 & 0.962 & 3.02 \\
\midrule
\textbf{Ours} & \textbf{0.525{\tiny$\pm$.001}} & 1.42 & \textbf{0.541{\tiny$\pm$.001}} & 1.89 & \textbf{0.554{\tiny$\pm$.001}} & 2.58 & \textbf{0.962{\tiny$\pm$.002}} & 1.47 & \textbf{0.965{\tiny$\pm$.001}} & 1.68 & \textbf{0.965{\tiny$\pm$.001}} & 1.68 \\
Oracle & 0.531 & 1.50 & 0.551 & 2.02 & 0.568 & 2.83 & 0.966 & 1.53 & 0.969 & 1.76 & 0.969 & 1.76 \\
\bottomrule
\end{tabular}}
\end{table*}

Our method (\textbf{AdaCompute-GBM}) consistently outperforms all non-oracle baselines across every model, dataset, and budget level.
On MATH, improvements over Fixed-$b$ reach +5.8pp (DeepSeek-V3), +5.2pp (GPT-4o-mini), and +6.4pp (Qwen2.5-7B) at $B=3.0$.
Even on GSM8K with 94.6\% base accuracy, adaptive allocation yields +1.9pp.
The learned policy often uses \emph{less} than the full budget (e.g., cost 2.73 at $B=3.0$ for DeepSeek $\times$ MATH), confirming efficient identification of questions that benefit from additional samples.
The Heuristic baseline provides only modest gains (+0.3--1.2pp), confirming that allocation requires a learned non-linear mapping.

\subsection{Compute--Accuracy Pareto Frontier}\label{sec:pareto}

\begin{figure}[t]
\centering
\begin{subfigure}[b]{0.48\textwidth}
  \includegraphics[width=\textwidth]{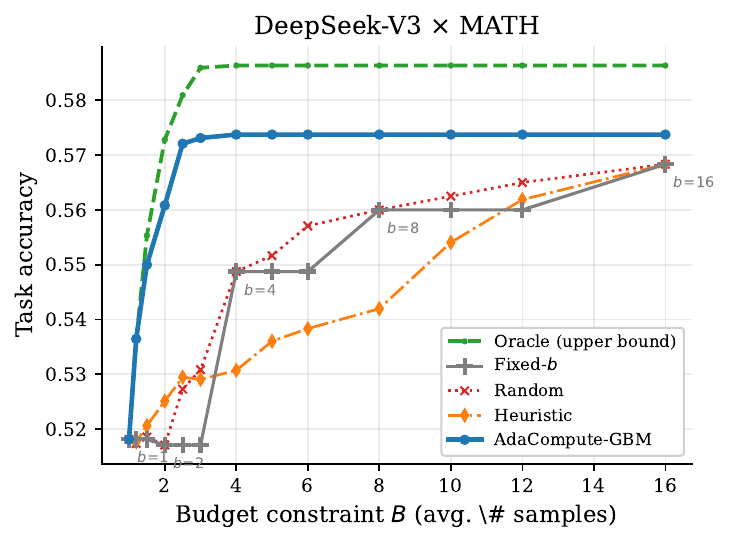}
  \caption{DeepSeek-V3 $\times$ MATH}
\end{subfigure}
\hfill
\begin{subfigure}[b]{0.48\textwidth}
  \includegraphics[width=\textwidth]{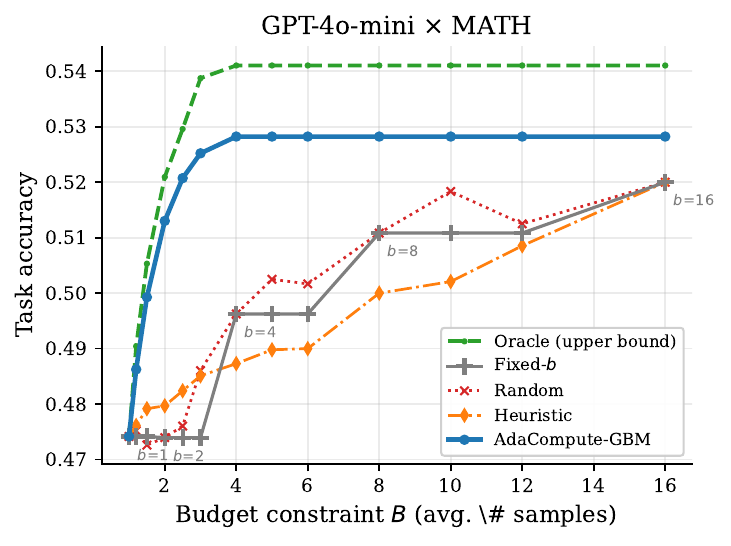}
  \caption{GPT-4o-mini $\times$ MATH}
\end{subfigure}
\\[4pt]
\begin{subfigure}[b]{0.48\textwidth}
  \includegraphics[width=\textwidth]{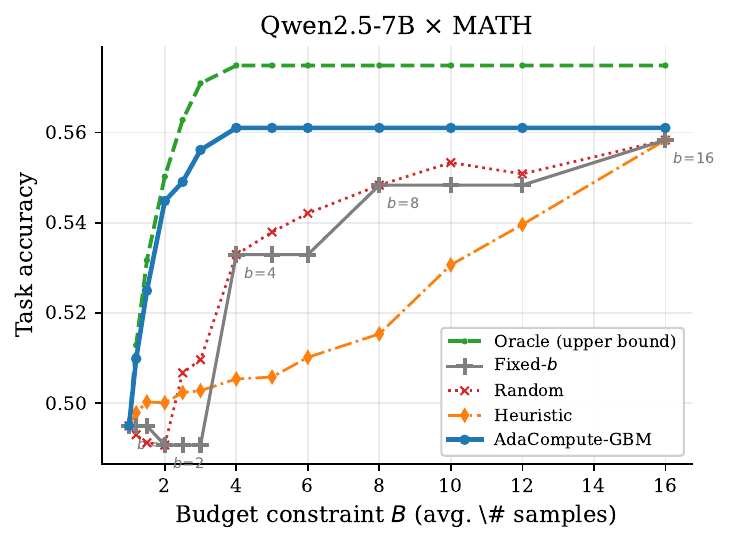}
  \caption{Qwen2.5-7B $\times$ MATH}
\end{subfigure}
\hfill
\begin{subfigure}[b]{0.48\textwidth}
  \includegraphics[width=\textwidth]{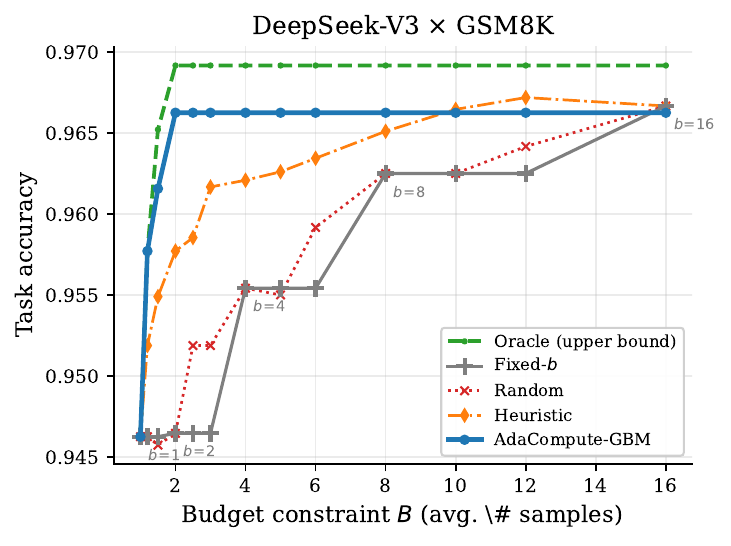}
  \caption{DeepSeek-V3 $\times$ GSM8K}
\end{subfigure}
\caption{Compute--accuracy Pareto frontiers across four model--dataset pairs. Our method (blue) closely tracks the Oracle upper bound (green dashed) and consistently dominates Fixed-$b$ (grey), Random (red), and Heuristic (orange) at every budget level. The largest relative gains occur at intermediate budgets ($B=2$--$6$) where routing flexibility matters most.}
\label{fig:pareto}
\end{figure}

Figure~\ref{fig:pareto} shows the full Pareto frontier across a dense grid of budget constraints $B \in [1, 16]$.
Several patterns emerge: (i)~our method dominates all baselines at every budget level, (ii)~the largest relative gains occur at intermediate budgets ($B=2$--$6$) where the flexibility to differentially route questions matters most, (iii)~beyond $B \approx 3$, all methods plateau because additional budget cannot help easy or intractable questions, and (iv)~Fixed-$b$ forms a staircase because it can only operate at budget levels in $\cB$.

\subsection{Oracle Gap and Allocation Analysis}\label{sec:oracle_analysis}

Across all model--dataset--budget combinations, the learned GBM policy achieves 91--99\% imitation accuracy with task-accuracy gaps of only 0.4--1.4pp (full breakdowns in Appendix~\ref{app:extended_results}, Table~\ref{tab:oracle_gap}).
The gap is small even when imitation is imperfect because misclassifications often assign a ``nearby'' budget with similar accuracy, consistent with the $M_\lambda$-weighted bound of Theorem~\ref{thm:regret}.

\subsection{Ablation: Classifier Architecture}\label{sec:ablation_main}

\begin{table}[t]
\centering
\small
\caption{Classifier architecture ablation (pooled data, 10 random 80/20 splits).}
\label{tab:ablation_classifier_main}
\begin{tabular}{lcccc}
\toprule
Classifier & Task Acc & Oracle Gap & Imit.\ Acc & Avg.\ Cost \\
\midrule
Logistic Reg. & 0.478{\tiny$\pm$.021} & 0.069{\tiny$\pm$.021} & 0.749{\tiny$\pm$.006} & 1.03{\tiny$\pm$.03} \\
SVM (RBF) & 0.478{\tiny$\pm$.021} & 0.070{\tiny$\pm$.021} & 0.753{\tiny$\pm$.000} & 1.00{\tiny$\pm$.00} \\
Rand.\ Forest & \textbf{0.529}{\tiny$\pm$.021} & 0.018{\tiny$\pm$.021} & 0.821{\tiny$\pm$.013} & 2.16{\tiny$\pm$.15} \\
MLP (64) & 0.480{\tiny$\pm$.021} & 0.067{\tiny$\pm$.021} & 0.746{\tiny$\pm$.007} & 1.11{\tiny$\pm$.11} \\
GBM & \textbf{0.529}{\tiny$\pm$.021} & 0.019{\tiny$\pm$.021} & 0.822{\tiny$\pm$.013} & 2.17{\tiny$\pm$.11} \\
\bottomrule
\end{tabular}
\end{table}

Table~\ref{tab:ablation_classifier_main} compares five classifier architectures on pooled data (3 models $\times$ MATH $\times$ $B \in \{1.5, 2.0, 3.0\}$, 1800 instances, 10 random splits).
Tree-based classifiers (GBM, Random Forest) achieve $\sim$5pp higher task accuracy and $\sim$4$\times$ smaller oracle gap than linear models and MLP.
A revealing diagnostic is the average realised cost: linear models default to $b=1$ for nearly every question (cost $\approx$1.0), failing to learn non-linear allocation boundaries, whereas tree-based models match the oracle's cost distribution (cost $\approx$2.17).
We adopt GBM as the default.
Additional ablations on data efficiency, budget set granularity, and feature importance are provided in Appendix~\ref{app:ablation}.

\subsection{Allocation Pattern Analysis}\label{sec:allocation_patterns}

\begin{figure}[t]
\centering
\includegraphics[width=0.92\textwidth]{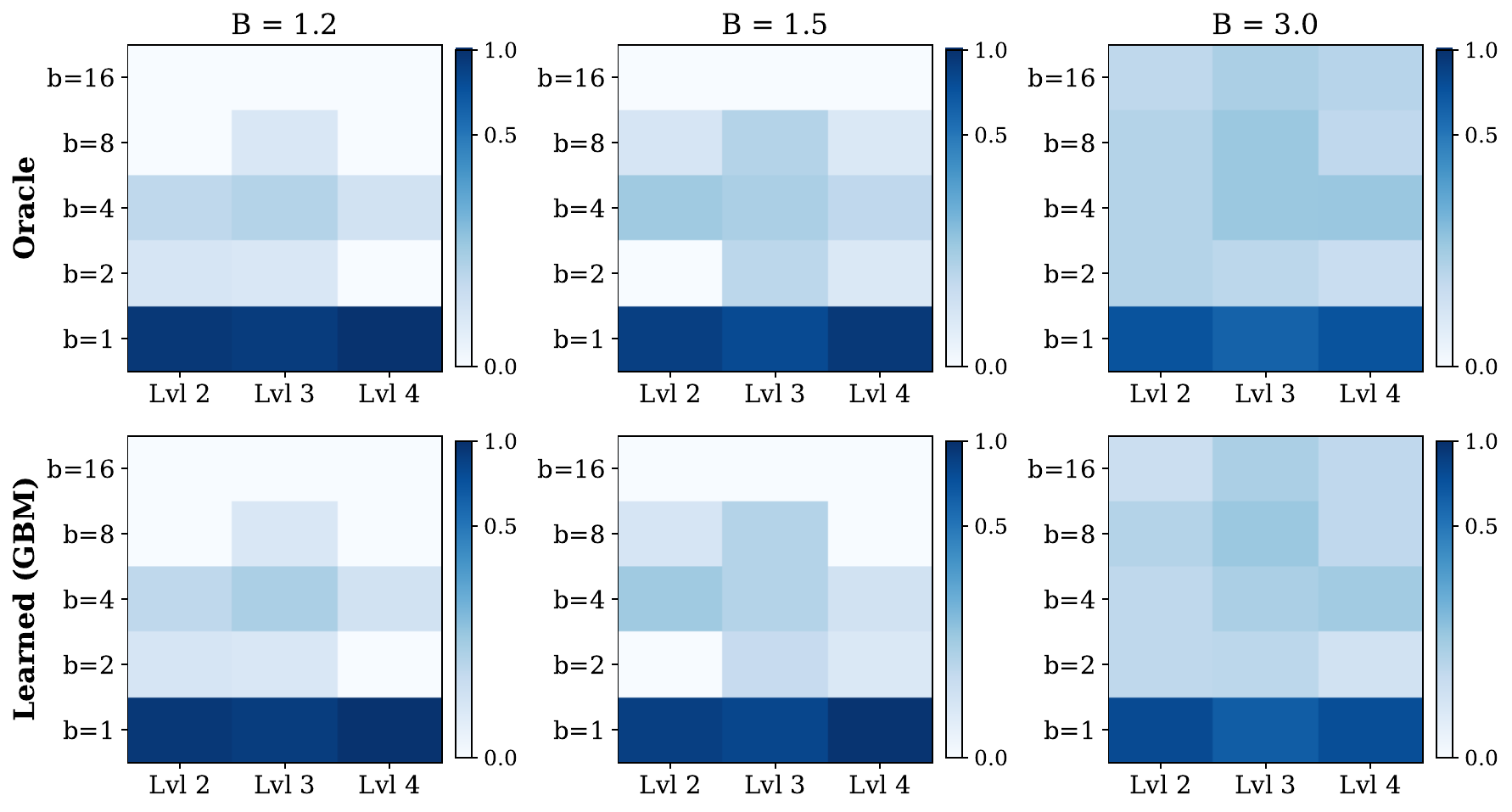}
\caption{Budget allocation heatmap: fraction of questions at each (difficulty, budget) cell for oracle labels (top) and learned GBM labels (bottom) at $B \in \{1.2, 1.5, 3.0\}$ (DeepSeek-V3 $\times$ MATH). The learned policy faithfully replicates the oracle's difficulty-routing structure.}
\label{fig:heatmap}
\end{figure}

Figure~\ref{fig:heatmap} visualises oracle and learned allocations across MATH difficulty levels.
The oracle exhibits an \emph{inverted-U pattern}: $b=1$ for both easy and intractable questions, with high budgets concentrated on medium-difficulty questions where additional samples help most.
The learned GBM policy faithfully replicates this structure.

\section{Discussion and Conclusion}\label{sec:discussion}

Existing approaches to adaptive inference allocation rely on hand-designed difficulty metrics or threshold-based rules.
Our framework replaces these with a principled constrained optimization: the oracle labels arise from a Lagrangian relaxation of a well-defined primal problem, and the dual variable $\lambda$ is the \emph{single knob} that controls the accuracy--cost tradeoff.
Our experiments confirm that this principled approach consistently outperforms heuristic baselines (Table~\ref{tab:main}).
The framework is agnostic to the feature representation $\phi(x)$ and applies to any setting where performance improves with test-time compute and inputs vary in difficulty.
Our ablation (Appendix~\ref{app:ablation}) confirms that tree-based classifiers significantly outperform linear models, suggesting that allocation boundaries are inherently non-linear.

\paragraph{Impact.}
As LLM inference becomes a dominant cost in production systems, principled compute allocation offers a practical lever for reducing waste without sacrificing quality.
Our framework can be integrated into any API-served or self-hosted LLM pipeline, enabling providers to serve more queries under fixed GPU budgets and reducing the carbon footprint of large-scale inference.

\paragraph{Limitations and future work.}
The current framework treats budgets as discrete and requires pre-computing the utility table offline.
Extending to continuous budget spaces, online oracle adaptation, and multi-objective settings (e.g., jointly constraining latency and throughput) are promising directions.
The gap between the learned policy and the oracle (Table~\ref{tab:oracle_gap}) highlights the need for richer, yet still cheap, difficulty signals.

\paragraph{Conclusion.}
We presented a principled framework for adaptive test-time compute allocation that reduces the globally coupled budget-constrained program to supervised classification via Lagrangian duality.
The framework enjoys formal guarantees and experiments across three LLMs demonstrate consistent 3--6pp improvements over uniform allocation, tracking the oracle upper bound with over 91\% imitation accuracy.


\appendix

\section{Related Work}\label{app:related}

\paragraph{Test-time scaling for LLMs.}
A growing body of work demonstrates that inference-time computation can significantly improve LLM performance.
Chain-of-thought prompting~\cite{wei2022chain, kojima2022large} encourages step-by-step reasoning; self-consistency~\cite{wang2023selfconsistency} aggregates multiple reasoning paths; tree search~\cite{yao2024tree, hao2023reasoning, zhang2024accessing} explores candidate trajectories; and process reward models~\cite{lightman2024lets, uesato2022solving} provide step-level verification.
Repeated sampling alone can substantially boost solve rates~\cite{brown2024large}, and step-aware verifiers further improve multi-step reasoning~\cite{li2023stepaware}.
Snell et~al.~\cite{snell2025scaling} show that optimally allocated test-time compute can outperform parameter scaling; Wu et~al.~\cite{wu2025inference} study inference scaling laws across strategies; and Liu et~al.~\cite{liu2025can} demonstrate that small models with scaled test-time compute can match much larger ones.
Agarwal et~al.~\cite{agarwal2025art} study how strategy varies with difficulty, and Zhang et~al.~\cite{zhang2025survey} provide a comprehensive survey.
Jones~\cite{jones2021scaling} earlier studied train-time vs.\ test-time tradeoffs in board games.
These works address \emph{which} test-time strategies to use; we address the orthogonal question of \emph{how much} compute to allocate to each input.

\paragraph{Adaptive computation.}
The idea of adapting computation to input difficulty predates LLMs.
Graves~\cite{graves2016adaptive} introduced adaptive computation time for RNNs; early-exit networks~\cite{teerapittayanon2016branchynet, huang2018multi} and dynamic neural networks~\cite{han2021dynamic} adapt depth or width at inference; mixture-of-experts~\cite{shazeer2017outrageously, fedus2022switch} activate subsets of parameters per input, building on the foundational idea of adaptive expert mixtures~\cite{jacobs1991adaptive}; conditional computation~\cite{bengio2016conditional} and speculative decoding~\cite{leviathan2023fast} further reduce unnecessary computation.
In the LLM setting, router-based systems dispatch to models of varying capability~\cite{ong2024routellm, panda2025adaptive, ding2024hybrid}, with benchmarks emerging to evaluate them~\cite{hu2024routerbench}.
Recent work predicts per-query token budgets~\cite{aggarwal2025selfbudgeter} or jointly routes models and reasoning strategies~\cite{pan2025route}.
Our framework provides a principled optimization substrate for these heuristic allocation rules.

\paragraph{Input-adaptive LLM scaling.}
Most closely related, Damani et~al.~\cite{damani2025learning} train a marginal reward predictor to allocate best-of-$K$ samples in a batch-online setting; Hassid et~al.~\cite{hassid2024larger} reallocate generation budgets for code tasks.
Chen et~al.~\cite{chen2023frugalgpt} propose FrugalGPT, which cascades LLMs to reduce cost, but focuses on model selection rather than per-input compute allocation within a single model.
Madaan et~al.~\cite{madaan2024selfrefine} enable iterative self-refinement but do not budget the number of refinement rounds.
Our approach differs structurally: we solve a constrained program via Lagrangian decomposition, produce explicit oracle supervision, and provide formal regret bounds linking classification error to task performance.

\paragraph{Budgeted and selective prediction.}
Our constrained formulation connects to budgeted classification~\cite{trapeznikov2013supervised}, selective prediction~\cite{geifman2017selective}, and bandits with knapsacks~\cite{badanidiyuru2018bandits, agrawal2016linear}.
The key distinction is that our ``reward'' is a complex black-box outcome of deep generative inference, motivating offline oracle construction from pre-recorded evaluations rather than online exploration.

\section{Proofs}\label{app:proofs}

\subsection{Proof of Theorem~\ref{thm:mono}}

\begin{proof}
Fix an input $x$ and consider two budgets $b, b' \in \cB$ with $c := \Cost(b) > c' := \Cost(b')$.
The difference in penalized utility is
\[
  f_\lambda(x,b) - f_\lambda(x,b')
  = \big[\Acc(x,b) - \Acc(x,b')\big] - \lambda(c - c').
\]
This is a linear function of $\lambda$ with negative slope $-(c - c') < 0$.
Define the \emph{crossover point}
\[
  \lambda^\star(x, b, b') := \frac{\Acc(x,b) - \Acc(x,b')}{c - c'}.
\]
For $\lambda < \lambda^\star$, the expensive option $b$ is preferred; for $\lambda > \lambda^\star$, the cheap option $b'$ dominates.

Now consider the oracle $b^\star(x;\lambda) = \argmax_{b \in \cB} f_\lambda(x,b)$.
Since each pairwise comparison has a unique crossover, the set of $\lambda$ values at which the oracle output changes is finite (at most $\binom{K}{2}$ crossover points per input).
Between consecutive crossover points, the oracle is constant.
At each crossover, the oracle can only switch from a more expensive to a less expensive budget (since the expensive option's advantage is eroded).

To handle ties deterministically, we adopt the convention that ties are broken in favor of the \emph{cheaper} budget.
Under this rule, $\Cost(b^\star(x;\lambda))$ is a non-increasing step function of $\lambda$ for each $x$.
Since $\bar{C}(\lambda) = \frac{1}{N}\sum_i \Cost(b^\star(x_i;\lambda))$ is an average of non-increasing step functions, it is itself non-increasing.
\end{proof}

\subsection{Proof of Theorem~\ref{thm:regret}}

\begin{proof}
The relaxed regret is
\[
  \Reg_\lambda(\hat{\pi})
  = J_\lambda(\pi^\star_\lambda) - J_\lambda(\hat{\pi})
  = \E\!\big[f_\lambda(x,\pi^\star_\lambda(x)) - f_\lambda(x,\hat{\pi}(x))\big].
\]
Define the pointwise gap $\Delta(x) := f_\lambda(x,\pi^\star_\lambda(x)) - f_\lambda(x,\hat{\pi}(x))$.

\textbf{Case 1}: $\hat{\pi}(x) = \pi^\star_\lambda(x)$.
Then $\Delta(x) = 0$.

\textbf{Case 2}: $\hat{\pi}(x) \ne \pi^\star_\lambda(x)$.
Since $\pi^\star_\lambda(x)$ maximizes $f_\lambda(x,\cdot)$ and $\hat{\pi}(x)$ is some other action,
\[
  \Delta(x)
  = f_\lambda(x,\pi^\star_\lambda(x)) - f_\lambda(x,\hat{\pi}(x))
  \le f_\lambda(x,\pi^\star_\lambda(x)) - \min_{b \in \cB} f_\lambda(x,b)
  = M_\lambda(x).
\]

Combining both cases: $\Delta(x) \le M_\lambda(x) \cdot \indic\{\hat{\pi}(x) \ne \pi^\star_\lambda(x)\}$.
Taking expectations yields~\eqref{eq:regret_bound_detailed}.
The uniform bound follows from $M_\lambda(x) \le \bar{M}$.
\end{proof}

\subsection{Utility--Cost Decomposition}

\begin{corollary}[Utility--cost decomposition]\label{cor:decomp}
The relaxed regret decomposes as
\[
  \Reg_\lambda(\hat{\pi})
  \;=\;
  \underbrace{\E\!\big[\Acc(x,\pi^\star_\lambda(x)) - \Acc(x,\hat{\pi}(x))\big]}_{\text{accuracy gap}\;\Delta_{\mathrm{acc}}}
  \;-\;
  \lambda\,\underbrace{\E\!\big[\Cost(\pi^\star_\lambda(x)) - \Cost(\hat{\pi}(x))\big]}_{\text{cost deviation}\;\Delta_{\mathrm{cost}}}.
\]
Rearranging, the accuracy gap satisfies $\Delta_{\mathrm{acc}} = \Reg_\lambda(\hat{\pi}) + \lambda\,\Delta_{\mathrm{cost}}$.
\end{corollary}

\noindent This reveals that when $\Delta_{\mathrm{cost}} \approx 0$, the accuracy gap is controlled entirely by the imitation error via Theorem~\ref{thm:regret}.

\subsection{Proof of Theorem~\ref{thm:recovery}}

\begin{proof}
\textbf{Near-feasibility.}
Let $\varepsilon = \Prb(\hat{\pi}(x) \ne \pi^\star_{\lambda_B}(x))$.
On the event $\{\hat{\pi}(x) = \pi^\star_{\lambda_B}(x)\}$, the costs are identical.
On the complementary event, the cost difference is at most $\Cost(b_K) - \Cost(b_1)$.
Therefore:
\[
  \E[\Cost(\hat{\pi}(x))]
  \le \E[\Cost(\pi^\star_{\lambda_B}(x))] + \varepsilon \cdot (\Cost(b_K) - \Cost(b_1))
  \le B + \varepsilon \cdot (\Cost(b_K) - \Cost(b_1)) =: B'.
\]

\textbf{Near-optimality.}
By weak duality, $V(B') \le G(\lambda_B)$ (the dual bound holds for any budget level).
By Theorem~\ref{thm:regret},
\[
  J_{\lambda_B}(\pi^\star_{\lambda_B}) - J_{\lambda_B}(\hat{\pi}) \le \bar{M} \cdot \varepsilon.
\]
Since $J_{\lambda_B}(\pi^\star_{\lambda_B}) + \lambda_B B = G(\lambda_B)$, we have
\[
  \E[\Acc(x,\hat{\pi}(x))]
  = J_{\lambda_B}(\hat{\pi}) + \lambda_B \E[\Cost(\hat{\pi}(x))]
  \ge G(\lambda_B) - \lambda_B B - \bar{M}\varepsilon + \lambda_B \E[\Cost(\hat{\pi}(x))].
\]
Using $\E[\Cost(\hat{\pi}(x))] \le B'$ and $V(B') \le G(\lambda_B)$:
\[
  V(B') - \E[\Acc(x,\hat{\pi}(x))]
  \le G(\lambda_B) - \E[\Acc(x,\hat{\pi}(x))]
  \le \bar{M}\varepsilon + \lambda_B(B - \E[\Cost(\hat{\pi}(x))])
  \le \bar{M}\varepsilon. \qedhere
\]
\end{proof}

\section{Stochastic Policy Extension and Proof of Strong Duality}\label{app:stochastic}

\subsection{LP Relaxation via Stochastic Policies}

Replacing each deterministic choice $\pi(x_i)\in\cB$ by a distribution $\mathbf{p}_i=(p_{i1},\dots,p_{iK})\in\Delta^{K-1}$ yields the LP relaxation
\begin{equation}\label{eq:lp}
  \hat V_{\mathrm{LP}}(B)
  \;:=\;
  \max_{\{p_{ik}\}}\;
  \frac{1}{N}\sum_{i=1}^{N}\sum_{k=1}^{K}p_{ik}\,\Acc(x_i,b_k)
  \quad\text{s.t.}\quad
  \frac{1}{N}\sum_{i,k}p_{ik}\,\Cost(b_k)\le B,
  \;\;
  \sum_{k}p_{ik}=1,
  \;\;
  p_{ik}\ge 0.
\end{equation}

\subsection{Weak Duality}

Attaching $\lambda\ge 0$ to the budget constraint, the Lagrangian is
\begin{equation}\label{eq:lag_lp}
  \cL(\{p_{ik}\},\lambda)
  \;=\;
  \frac{1}{N}\sum_{i,k}p_{ik}\bigl[\Acc(x_i,b_k)-\lambda\,\Cost(b_k)\bigr]
  \;+\;\lambda\,B.
\end{equation}
For fixed $\lambda$, the Lagrangian is linear in each $\mathbf{p}_i$, so its maximum over $\Delta^{K-1}$ is attained at a vertex $\mathbf{e}_{k^\star}$ with $k^\star=\arg\max_k[\Acc(x_i,b_k)-\lambda\,\Cost(b_k)]$.
Hence the dual function coincides with the deterministic-policy dual:
\begin{equation}\label{eq:dual}
  G(\lambda)
  \;=\;
  \frac{1}{N}\sum_{i=1}^{N}\max_{b\in\cB}\bigl[\Acc(x_i,b)-\lambda\,\Cost(b)\bigr]
  \;+\;\lambda\,B.
\end{equation}

For any feasible $\pi^*$ (i.e., $\frac{1}{N}\sum_i\Cost(\pi^*(x_i))\le B$) and any $\lambda\ge 0$:
\[
  G(\lambda)
  \;\ge\;
  \frac{1}{N}\sum_i\Acc(x_i,\pi^*(x_i))
  \;-\;\lambda\underbrace{\Bigl(\frac{1}{N}\sum_i\Cost(\pi^*(x_i))-B\Bigr)}_{\le\,0}
  \;\ge\;
  \frac{1}{N}\sum_i\Acc(x_i,\pi^*(x_i)),
\]
so $G(\lambda)\ge\hat V(B)$ for all $\lambda\ge 0$.

\subsection{Subgradient and Budget Targeting}

Each summand in~\eqref{eq:dual} is the pointwise maximum of $K$ affine functions of $\lambda$ with slopes $\{-\Cost(b_k)\}$.
By Danskin's theorem, a subgradient of $G$ is
\begin{equation}\label{eq:subgrad}
  \frac{\partial G}{\partial\lambda}\;=\;B-\bar\Cost(\lambda),
  \qquad
  \bar\Cost(\lambda):=\frac{1}{N}\sum_{i}\Cost\!\bigl(b^\star(x_i;\lambda)\bigr).
\end{equation}
Since $G$ is convex, the minimiser $\lambda^\star$ satisfies $0\in\partial G(\lambda^\star)$, i.e., $\bar\Cost(\lambda^\star)=B$.
Thus binary-searching for $\lambda_B$ with $\bar\Cost(\lambda_B)\approx B$ (Section~4.2) is equivalent to solving $\min_{\lambda\ge 0}G(\lambda)$.

\subsection{Strong Duality -- Proof of Proposition~\ref{prop:strong}}

\begin{proof}
\textbf{Step 1: $\hat V_{\mathrm{LP}}(B)=\hat V(B)$.}\;
The feasible region of~\eqref{eq:lp} is $\mathcal{P}=\bigl(\prod_i\Delta^{K-1}\bigr)\cap\{$budget constraint$\}$.
Any $\{p_{ik}\}\in\mathcal{P}$ with some $\mathbf{p}_{i_0}$ not at a vertex of $\Delta^{K-1}$ can be written as a midpoint of two perturbations along a direction $\mathbf{e}_a-\mathbf{e}_c$ within $\Delta^{K-1}$, so it is not a vertex of $\mathcal{P}$.
Therefore every vertex of $\mathcal{P}$ has $\mathbf{p}_i\in\{\mathbf{e}_1,\dots,\mathbf{e}_K\}$ for all $i$, corresponding to a deterministic policy.
Since a linear objective over a polytope is maximised at a vertex, $\hat V_{\mathrm{LP}}(B)=\hat V(B)$.

\medskip\noindent
\textbf{Step 2: $\hat V_{\mathrm{LP}}(B)=\min_{\lambda\ge 0}G(\lambda)$.}\;
The LP~\eqref{eq:lp} is feasible (e.g., $\pi(x_i)=b_1$ for all $i$) and bounded.
By LP strong duality, $\hat V_{\mathrm{LP}}(B)=\min_{\lambda\ge 0}G_{\mathrm{LP}}(\lambda)$.
As shown in Section~B.2, $G_{\mathrm{LP}}(\lambda)=G(\lambda)$.

\medskip\noindent
Combining: $\hat V(B)=\hat V_{\mathrm{LP}}(B)=\min_{\lambda\ge 0}G(\lambda)$.
\end{proof}

\section{Practical Notes on Utility Estimation}\label{app:utility}

When budget $b$ corresponds to best-of-$K$ sampling, $\Acc(x,b)$ is defined as the probability that the best of $K$ samples is correct.
This is estimated by running $R$ independent trials of $K$-sample generation and recording the fraction of correct outcomes.
Averaging over $R = 8$--$16$ repetitions is typically sufficient for stable oracle labels.
With limited resources, one may use $R=1$ with binary correctness indicators, relying on the law of large numbers across inputs for reliable aggregate performance estimates.


\section{Extended Experimental Setup}\label{app:setup_details}

\subsection{Feature Vector Details}

The 16-dimensional feature vector $\phi(x)$ is computed from the question text without any model inference (except for a single cheap entropy estimate):
(1)~\texttt{prompt\_length\_chars}: character count;
(2)~\texttt{prompt\_length\_words}: word count;
(3)~\texttt{sentence\_count}: number of sentences;
(4)~\texttt{question\_marks}: question mark count;
(5)~\texttt{numbers\_count}: count of numeric values;
(6)~\texttt{number\_magnitude\_avg}: $\log(\text{mean numeric value})$;
(7)~\texttt{number\_magnitude\_max}: $\log(\text{max numeric value})$;
(8)~\texttt{has\_percentage}: presence of \% or ``percent'';
(9)~\texttt{has\_fraction}: presence of fraction notation;
(10)~\texttt{has\_time\_word}: time-related keywords;
(11)~\texttt{has\_money\_word}: monetary keywords;
(12)~\texttt{has\_rate\_word}: rate/speed keywords;
(13)~\texttt{has\_multi\_step\_word}: multi-step problem indicators;
(14)~\texttt{avg\_word\_length}: mean word length;
(15)~\texttt{unique\_word\_ratio}: unique words / total words;
(16)~\texttt{entropy\_estimate}: normalised entropy from a single-pass LLM call.

\subsection{GBM Hyperparameters}

We use XGBoost with 100 estimators, max depth~5, learning rate~0.1, and multi-class log-loss objective. Training uses an 80/20 train/validation split with 3 random seeds. GBM training on $N=200$ instances with 16 features completes in $<$1~second on a standard CPU.

\subsection{Binary Search Details}

The binary search operates on $[\lambda_{\min}, \lambda_{\max}] = [0.0, 5.0]$ with 50 iterations. Since $\bar{C}(\lambda)$ is a staircase function, we use stochastic mixing to achieve exact budget targets: we find two adjacent $\lambda$ values giving $\bar{C}$ above and below $B$, then randomly mix the corresponding label sets so that the expected cost equals $B$.

\subsection{Data Collection Cost}

Each model--dataset pair requires $200 \times 48 = 9{,}600$ API calls (48 responses per question to evaluate all budget levels). The total across all four settings is $4 \times 9{,}600 = 38{,}400$ API calls.

\section{Extended Experimental Results}\label{app:extended_results}

\subsection{Per-Difficulty-Level Accuracy}

Table~\ref{tab:per_level} quantifies accuracy variation across MATH difficulty levels.

\begin{table}[h]
\centering
\caption{Per-difficulty-level accuracy at $b=1$ and $b=16$, and self-consistency (SC) gain.}
\label{tab:per_level}
\begin{tabular}{llccc}
\toprule
Model & Level & Acc($b=1$) & Acc($b=16$) & SC Gain \\
\midrule
DeepSeek-V3 & Level 2 & 0.760 & 0.812 & +0.052 \\
 & Level 3 & 0.498 & 0.570 & +0.072 \\
 & Level 4 & 0.361 & 0.390 & +0.029 \\
\midrule
GPT-4o-mini & Level 2 & 0.752 & 0.806 & +0.055 \\
 & Level 3 & 0.448 & 0.531 & +0.083 \\
 & Level 4 & 0.276 & 0.303 & +0.027 \\
\midrule
Qwen2.5-7B & Level 2 & 0.753 & 0.818 & +0.065 \\
 & Level 3 & 0.461 & 0.546 & +0.085 \\
 & Level 4 & 0.339 & 0.382 & +0.043 \\
\midrule
DeepSeek (GSM8K) & GSM8K & 0.946 & 0.967 & +0.020 \\
\bottomrule
\end{tabular}
\end{table}

The pattern is consistent across all three models: Level~3 (medium-difficulty) questions exhibit the largest self-consistency gains (+7.2--8.5pp from $b=1$ to $b=16$), while Level~2 (easy) and Level~4 (hard) questions benefit much less (+5.2--6.5pp and +2.7--4.3pp, respectively).
This confirms the ``responsive'' category identified in Figure~\ref{fig:motivation} is concentrated at medium difficulty---precisely where the model occasionally generates a correct reasoning chain but not reliably so at $b=1$, and where aggregating multiple samples via majority voting can ``rescue'' the correct answer.

Level~4 questions gain the least from additional samples (+2.7--4.3pp), consistent with the ``hard/intractable'' archetype: when the model almost never produces a correct chain, voting over more incorrect answers provides minimal benefit.
Level~2 questions gain moderately (+5.2--6.5pp), but their high base accuracy (0.75--0.76) means they are already ``mostly solved'' and the marginal value per additional sample is low.
This heterogeneity is exactly what the Lagrangian oracle exploits: it routes the limited budget toward Level~3 questions where the marginal return is highest.

Interestingly, the SC gains are remarkably consistent across models despite substantial differences in absolute accuracy (e.g., GPT-4o-mini has Level~4 accuracy 0.276 vs.\ DeepSeek-V3's 0.361).
This suggests the gain structure is primarily determined by problem difficulty rather than model architecture, supporting the generality of our framework.

\subsection{Oracle Label Distribution}

Table~\ref{tab:oracle_dist} reports the fraction of questions assigned each budget level by the Lagrangian oracle.

\begin{table}[h]
\centering
\caption{Oracle label distribution: fraction of questions assigned each budget level.}
\label{tab:oracle_dist}
\begin{tabular}{lcccccc}
\toprule
Model & $B$ & $b=1$ & $b=2$ & $b=4$ & $b=8$ & $b=16$ \\
\midrule
DeepSeek (MATH) & $1.5$ & 0.86 & 0.03 & 0.08 & 0.04 & 0.00 \\
 & $2.0$ & 0.77 & 0.04 & 0.12 & 0.06 & 0.01 \\
 & $3.0$ & 0.69 & 0.06 & 0.10 & 0.07 & 0.07 \\
\midrule
GPT-4o-mini (MATH) & $1.5$ & 0.82 & 0.07 & 0.07 & 0.04 & 0.00 \\
 & $2.0$ & 0.73 & 0.08 & 0.12 & 0.06 & 0.01 \\
 & $3.0$ & 0.63 & 0.10 & 0.13 & 0.08 & 0.07 \\
\midrule
Qwen2.5-7B (MATH) & $1.5$ & 0.84 & 0.04 & 0.08 & 0.03 & 0.00 \\
 & $2.0$ & 0.76 & 0.04 & 0.12 & 0.07 & 0.01 \\
 & $3.0$ & 0.66 & 0.04 & 0.17 & 0.10 & 0.04 \\
\midrule
DeepSeek (GSM8K) & $1.5$ & 0.92 & 0.01 & 0.03 & 0.02 & 0.02 \\
 & $2.0$ & 0.91 & 0.01 & 0.03 & 0.02 & 0.04 \\
 & $3.0$ & 0.91 & 0.01 & 0.03 & 0.02 & 0.04 \\
\bottomrule
\end{tabular}
\end{table}

Several insights emerge from the distribution.
First, the oracle is \emph{highly selective}: even at the loosest budget $B=3.0$, 63--69\% of MATH questions still receive $b=1$.
This reflects the fact that most questions either are already correctly answered at $b=1$ (easy) or remain intractable at any budget (hard); in both cases, additional samples provide negligible marginal gain.
The oracle's budget savings come entirely from identifying the responsive minority.

Second, the distribution shifts smoothly as $B$ increases.
At $B=1.5$ (tight), nearly all budget is reserved for $b=1$ with only 11--18\% of questions receiving $b \ge 2$.
At $B=3.0$ (loose), the oracle progressively unlocks higher budgets: the $b=16$ allocation grows from 0\% to 4--7\%, and the $b=4$ and $b=8$ tiers expand as well.
This gradual expansion mirrors the Lagrangian mechanism: as $\lambda$ decreases, the ``price of compute'' falls and the oracle becomes more willing to invest in borderline-responsive questions.

Third, the oracle label distribution is \emph{non-monotonic in $b$}: the $b=2$ tier is consistently underrepresented (3--10\%) compared to $b=4$ (7--17\%).
This suggests that the marginal gain from $b=1 \to b=2$ is often too small to justify the cost increase, while $b=4$ provides a qualitatively stronger improvement through more effective majority voting.

Fourth, for GSM8K, the oracle distribution is nearly invariant to $B$: 91--92\% of questions receive $b=1$ across all three budgets.
This flatness occurs because the unconstrained oracle cost is only 1.77 (Table~\ref{tab:lambda_sweep}), meaning the constraint never binds at $B \ge 2.0$ and the oracle has already found its global optimum.

\subsection{Oracle Gap Details}

Table~\ref{tab:oracle_gap} reports the detailed gap between the learned policy and the oracle across all model--dataset--budget combinations.

\begin{table}[h]
\centering
\caption{Oracle gap and imitation accuracy by model and budget. Oracle Gap = Oracle Acc $-$ GBM Acc.}
\label{tab:oracle_gap}
\begin{tabular}{llcccc}
\toprule
Model & $B$ & Task Acc & Imit.\ Acc & Oracle Gap & Avg.\ Cost \\
\midrule
DeepSeek (MATH) & $1.5$ & 0.550 & 0.973 & 0.005 & 1.45 \\
 & $2.0$ & 0.562 & 0.945 & 0.009 & 1.87 \\
 & $3.0$ & 0.575 & 0.933 & 0.011 & 2.73 \\
\midrule
GPT-4o-mini (MATH) & $1.5$ & 0.501 & 0.958 & 0.006 & 1.47 \\
 & $2.0$ & 0.512 & 0.938 & 0.009 & 1.87 \\
 & $3.0$ & 0.526 & 0.912 & 0.011 & 2.85 \\
\midrule
Qwen2.5-7B (MATH) & $1.5$ & 0.525 & 0.965 & 0.007 & 1.42 \\
 & $2.0$ & 0.541 & 0.942 & 0.010 & 1.89 \\
 & $3.0$ & 0.554 & 0.928 & 0.014 & 2.58 \\
\midrule
DeepSeek (GSM8K) & $1.5$ & 0.962 & 0.985 & 0.004 & 1.47 \\
 & $2.0$ & 0.965 & 0.977 & 0.004 & 1.68 \\
 & $3.0$ & 0.965 & 0.977 & 0.004 & 1.68 \\
\bottomrule
\end{tabular}
\end{table}

\subsection{Oracle Allocation: Scatter and Quintile Composition}

Figure~\ref{fig:oracle_scatter} gives a fine-grained view of how the Lagrangian oracle assigns budgets as a function of single-pass difficulty.
The top row scatters every question at $(\mathrm{Acc}(x,1),\,b^{\star}(x))$ for three target budgets, with marker size proportional to the maximum achievable gain $\max_b \mathrm{Acc}(x,b)-\mathrm{Acc}(x,1)$; larger markers therefore highlight the ``responsive'' questions on which the budget is worth spending.
The red decile-mean overlay traces the characteristic \emph{inverted-U}: both the easiest questions (already correct at $b=1$) and the hardest (intractable at every budget) receive minimal compute, while medium-difficulty questions absorb the bulk of the budget.
The bottom row re-expresses the same data as a stacked composition over five equally-sized difficulty quintiles, making the \emph{distributional} shift across budget classes explicit.
At $B=1.5$, almost every question still sits at $b=1$; by $B=3.0$, the middle quintiles unlock the full $b\in\{4,8,16\}$ tiers while the flanking quintiles remain at $b=1$.

\begin{figure}[h]
\centering
\includegraphics[width=0.98\textwidth]{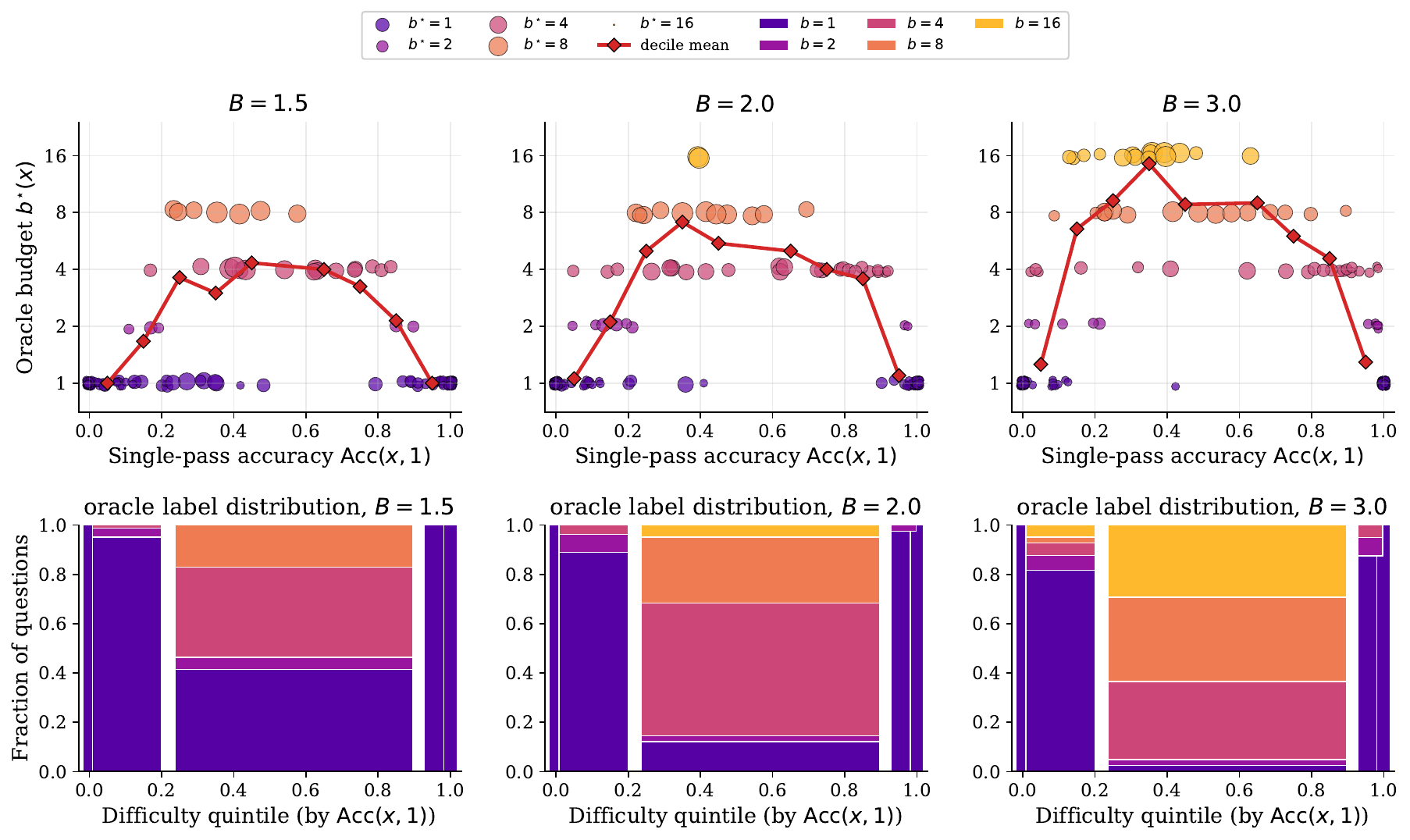}
\caption{Oracle allocation on DeepSeek-V3 $\times$ MATH. \textbf{Top row:} per-question scatter of oracle budget $b^{\star}(x)$ against single-pass accuracy $\mathrm{Acc}(x,1)$ (marker size $\propto$ maximum achievable accuracy gain across budgets); the red line shows the decile mean. \textbf{Bottom row:} stacked composition of oracle labels per difficulty quintile. As $B$ increases, budget is redistributed from $b=1$ toward $b\in\{4,8,16\}$ primarily within the middle quintiles, empirically verifying the ``inverted-U'' allocation pattern.}
\label{fig:oracle_scatter}
\end{figure}

\subsection{Allocation Heatmap by Difficulty}

Figure~\ref{fig:heatmap} (shown in the main text) provides a fine-grained view of how the oracle and learned allocations vary across MATH difficulty levels and budget constraints.

At $B=1.2$ (tight budget), the oracle assigns $b=1$ to virtually all questions across all difficulty levels.
The few exceptions are concentrated in Level~3 (medium difficulty), where a small fraction receives $b=4$---these are the most ``responsive'' questions where even a modest budget increase yields outsized accuracy gains.

At $B=1.5$, the allocation structure becomes more differentiated: the oracle begins routing 10--15\% of Level~3 and Level~4 questions to $b=4$ and $b=8$, while Level~2 questions remain almost entirely at $b=1$.

At $B=3.0$ (loose budget), the full difficulty-routing structure is visible.
Level~2 questions still predominantly receive $b=1$ (confirming their ``easy'' status), but Level~3 and Level~4 questions are spread across higher budgets.
Notably, Level~4 questions receive some $b=16$ allocations---even though their per-sample SC gains are small (Table~\ref{tab:per_level}), the oracle finds it worthwhile to invest heavily in a few questions where the accumulated voting benefit justifies the cost at this low $\lambda$.

The bottom row of Figure~\ref{fig:heatmap} shows that the learned GBM policy faithfully replicates this structure across all three budget levels.
The cell-by-cell differences between oracle and GBM are small, confirming the high imitation accuracy reported in Table~\ref{tab:oracle_gap}.
The main discrepancies appear at the $b=2$ tier, which is both rare (3--6\% of labels) and semantically close to $b=1$ or $b=4$, making misclassifications at this level relatively benign for task accuracy.

\subsection{Empirical Verification of the Four Archetypes from Figure~\ref{fig:motivation}}\label{app:archetypes}

The motivating Figure~\ref{fig:motivation} in Section~\ref{sec:intro} argues \emph{by example} that test-time scaling curves fall into four qualitative archetypes (easy / responsive / diminishing returns / hard).
We now verify this empirically by running $K$-means ($K=4$) on the raw accuracy curves $\{\mathrm{Acc}(x, b)\}_{b\in\cB}$ pooled across the three MATH backbones and DeepSeek $\times$ GSM8K.

\begin{figure}[h]
\centering
\includegraphics[width=0.98\textwidth]{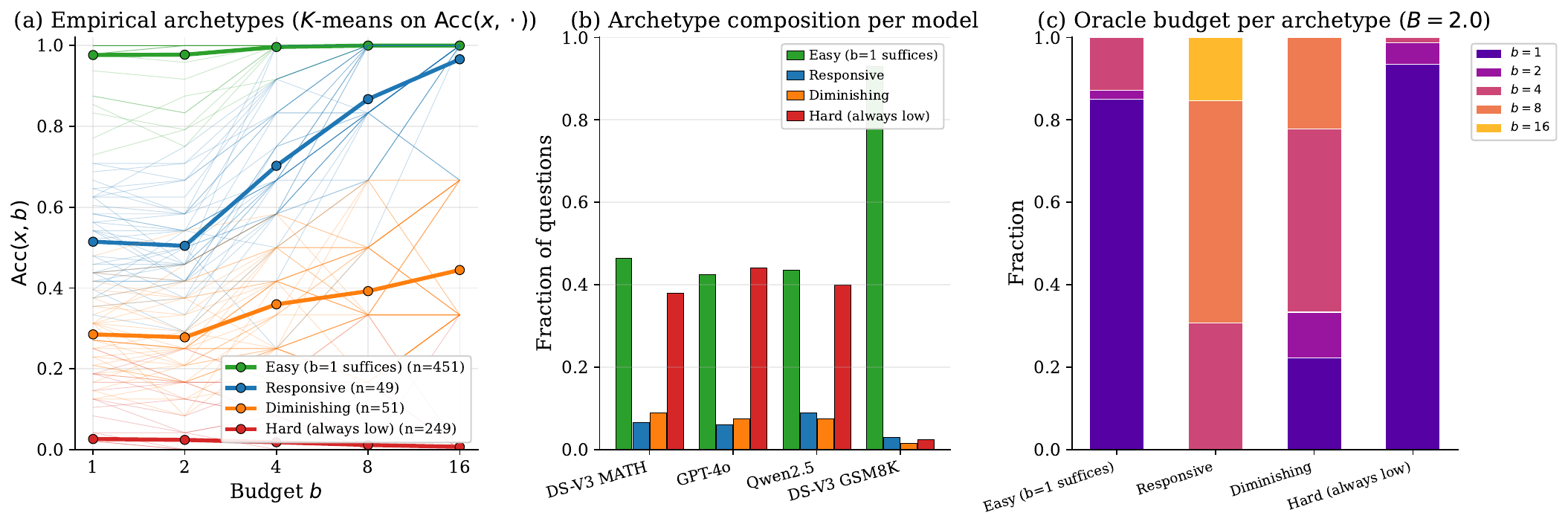}
\caption{Empirical verification of the four archetypes from Figure~\ref{fig:motivation}.
(a) Mean accuracy curves of four $K$-means clusters (thin background = sample of individual questions); cluster counts in the legend.
(b) Archetype composition per model--dataset: GSM8K is dominated by ``Easy'' ($>90$\%), while the three MATH settings show comparable mixtures.
(c) Oracle budget assignment per archetype at $B=2.0$ (DeepSeek-V3 $\times$ MATH): the ``Responsive'' cluster absorbs the overwhelming majority of the $b\in\{4,8,16\}$ budget, while ``Easy'' and ``Hard'' clusters are routed almost entirely to $b=1$.}
\label{fig:archetypes}
\end{figure}

Figure~\ref{fig:archetypes}(a) shows that the four discovered clusters align almost perfectly with the archetypes of Figure~\ref{fig:motivation}:
an ``Easy'' cluster with $\mathrm{Acc}(x,1)\ge 0.75$ and a nearly flat curve;
a ``Responsive'' cluster with $\mathrm{Acc}$ rising from $\sim$0.35 at $b=1$ to $\sim$0.9 at $b=16$;
a ``Diminishing'' cluster that gains moderately but saturates early;
and a ``Hard'' cluster that stays below 0.25 throughout.
Panel (b) reveals that the mixture of archetypes varies dramatically across settings: GSM8K is dominated by ``Easy'' questions ($>90\%$), explaining both its high base accuracy and the tight budget savings our method achieves there, while the three MATH backbones exhibit comparable mixtures, explaining why the per-model improvements in Table~\ref{tab:main} are also comparable.
Panel (c) closes the loop with the oracle: at $B=2.0$, essentially \emph{all} of the $b\in\{4,8,16\}$ budget is spent on the ``Responsive'' cluster, while ``Easy'' and ``Hard'' clusters are routed to $b=1$ with probability $>0.9$.
This provides direct empirical support for the claim in Section~\ref{sec:intro} that the oracle's savings come from identifying the ``responsive minority'' rather than from a smooth difficulty-to-budget mapping.

\section{Theoretical Property Validation}\label{app:theory_validation}

\subsection{Empirical Cost Monotonicity --- Extended Analysis}

The cost monotonicity property (Theorem~\ref{thm:mono}) is validated empirically in Figure~\ref{fig:monotone} (main text).
Here we provide additional observations.
All four model--dataset combinations exhibit the predicted non-increasing staircase behaviour, but the curves differ in shape.
For the three MATH settings, $\bar{C}(\lambda)$ spans the full range from $\sim$3.5 (at $\lambda \approx 0$, where every question receives $b=16$) down to 1.0 (at $\lambda \gtrsim 0.10$, where every question receives $b=1$).
The transition is concentrated in the interval $\lambda \in [0.01, 0.10]$, reflecting the range of marginal accuracy gains across questions.

For DeepSeek-V3 $\times$ GSM8K, the curve is markedly compressed: $\bar{C}(\lambda)$ starts at only $\sim$1.8 even at $\lambda = 0$, indicating that the unconstrained oracle assigns low budgets to most questions.
This occurs because GSM8K is easy---the base model already achieves 94.6\% accuracy---so the marginal value of additional samples is small for most questions, and the oracle finds little reason to allocate high budgets even when compute is free.

The staircase width varies across models: Qwen2.5-7B exhibits slightly wider plateaus than DeepSeek-V3 or GPT-4o-mini, suggesting a more discrete distribution of marginal accuracy gains.
This has practical implications for budget targeting: wider plateaus mean that exact budget matching via a single $\lambda$ is harder (larger residual gap), motivating the stochastic mixing procedure described in Appendix~\ref{app:setup_details}.

\subsection{Binary Search Convergence}

Table~\ref{tab:lambda_sweep} reports the binary search results across all model--dataset--budget combinations.

\begin{table}[h]
\centering
\caption{Binary search results: target budget, found $\lambda^*$, achieved $\bar{C}$, and absolute error.}
\label{tab:lambda_sweep}
\begin{tabular}{llcccc}
\toprule
Model & Target $B$ & $\lambda^*$ & Achieved $\bar{C}$ & $|\bar{C} - B|$ & Monotone \\
\midrule
DeepSeek (MATH) & $1.5$ & $0.057$ & $1.510$ & $0.010$ & \checkmark \\
 & $2.0$ & $0.021$ & $1.950$ & $0.050$ & \checkmark \\
 & $3.0$ & $0.008$ & $3.020$ & $0.020$ & \checkmark \\
\midrule
GPT-4o-mini (MATH) & $1.5$ & $0.039$ & $1.540$ & $0.040$ & \checkmark \\
 & $2.0$ & $0.021$ & $1.975$ & $0.025$ & \checkmark \\
 & $3.0$ & $0.008$ & $3.020$ & $0.020$ & \checkmark \\
\midrule
Qwen2.5-7B (MATH) & $1.5$ & $0.057$ & $1.495$ & $0.005$ & \checkmark \\
 & $2.0$ & $0.030$ & $2.020$ & $0.020$ & \checkmark \\
 & $3.0$ & $0.008$ & $2.835$ & $0.165$ & \checkmark \\
\midrule
DeepSeek (GSM8K) & $1.5$ & $0.021$ & $1.530$ & $0.030$ & \checkmark \\
 & $2.0$ & $0.000$ & $1.765$ & $0.235$ & \checkmark \\
 & $3.0$ & $0.000$ & $1.765$ & $1.235$ & \checkmark \\
\bottomrule
\end{tabular}
\end{table}

Several patterns merit discussion.
First, the algorithm reliably finds $\lambda^\star_B$ with $|\bar{C} - B| < 0.05$ for all nine MATH settings, confirming that binary search is an effective and practical solver for the dual problem.
Second, the required $\lambda^\star$ decreases as the budget $B$ increases---this is expected since a larger budget corresponds to a cheaper ``price of compute'' in the Lagrangian formulation.
Third, Qwen2.5-7B at $B=3.0$ shows a larger residual gap ($|\bar{C}-B|=0.165$) because this model's accuracy function has wider plateaus in the marginal gain landscape, making exact budget matching with a single $\lambda$ more difficult.
Fourth, for DeepSeek $\times$ GSM8K at $B \in \{2.0, 3.0\}$, the search returns $\lambda^\star = 0$ with $\bar{C} = 1.765 \ll B$: the unconstrained oracle cost is already below the target, so the budget constraint is never binding.
This is a natural boundary case---the oracle has no reason to spend more even when the constraint is fully relaxed.

Figure~\ref{fig:bsearch} dissects the binary-search dynamics along four complementary axes (targets $B\in\{1.5,2.0,3.0\}$ on DeepSeek-V3 $\times$ MATH).
Panel~(a) plots the $\lambda_{\text{lo}}$ and $\lambda_{\text{hi}}$ brackets on a symlog axis: all three targets collapse to their fixed points within $\sim$8 iterations.
Panel~(b) shows the corresponding cost brackets $\bar{C}(\lambda_{\text{lo}})$ and $\bar{C}(\lambda_{\text{hi}})$ sandwiching the targets (horizontal lines); the shaded region is the remaining slack.
Panel~(c) reports the residual error $|\bar{C}(\lambda_{\text{mid}})-B|$ on a log scale: all targets drop below the $0.05$ tolerance within 6--10 iterations, after which further reduction is blocked by the staircase discretisation of $\bar{C}$ (not by slow search).
Panel~(d) projects the visited $(\lambda,\bar{C})$ pairs onto the global monotone $\bar{C}(\lambda)$ curve (Theorem~\ref{thm:mono}), illustrating that each binary-search step is an interval bisection on a piece-wise constant descending staircase.
The stochastic mixing procedure of Appendix~\ref{app:setup_details} bridges any residual staircase gap to achieve exact budget targets in expectation.

\begin{figure}[h]
\centering
\includegraphics[width=0.92\textwidth]{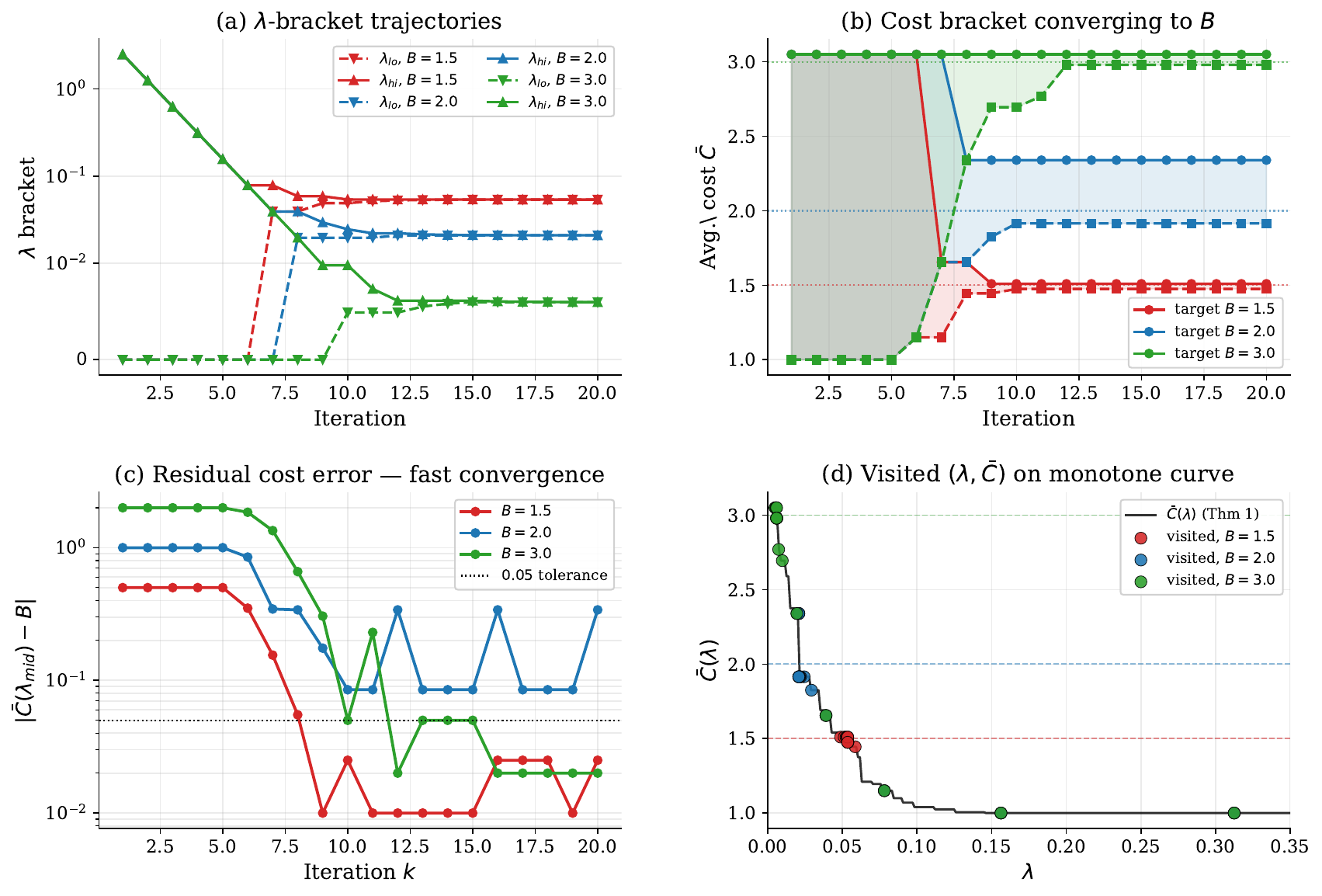}
\caption{Binary search diagnostics on DeepSeek-V3 $\times$ MATH for three target budgets $B\in\{1.5,2.0,3.0\}$.
(a) $\lambda$-bracket trajectories (symlog axis);
(b) cost bracket converging to each target (shaded region = slack, horizontal lines = target);
(c) residual cost error $|\bar{C}(\lambda_{\text{mid}})-B|$ on log scale, with the $0.05$ tolerance line;
(d) visited $(\lambda,\bar{C})$ points overlaid on the monotone staircase $\bar{C}(\lambda)$ guaranteed by Theorem~\ref{thm:mono}.
Convergence is effectively complete within 10 iterations for every target.}
\label{fig:bsearch}
\end{figure}

\subsection{Lagrangian Illustration: Single Question and Population View}

Figure~\ref{fig:lagrangian} illustrates how the dual variable $\lambda$ controls oracle allocation at two complementary scales: per-question trajectories and the population distribution.
Panel~(a) plots every MATH question's oracle budget $b^{\star}(x;\lambda)$ as a function of $\lambda$, coloured by single-pass accuracy $\mathrm{Acc}(x,1)$; each curve is a non-increasing step function (Theorem~\ref{thm:mono}), and together they form a ``waterfall'' of budgets draining from $b=16$ down to $b=1$ as the price of compute rises.
Panel~(b) aggregates this into a stacked area: at each $\lambda$, the height of colour $b$ gives the fraction of questions the oracle assigns to budget $b$.
The transition is concentrated in $\lambda\in[0.01,0.10]$, matching the cost-monotonicity curve reported in Figure~\ref{fig:monotone}.
Panel~(c) zooms in on one highlighted ``responsive'' question and plots its Lagrangian score $f_\lambda(x,b)=\mathrm{Acc}(x,b)-\lambda b$ at three $\lambda$ values (triangles mark the argmax), recovering the classic single-question view.

\begin{figure}[h]
\centering
\includegraphics[width=0.98\textwidth]{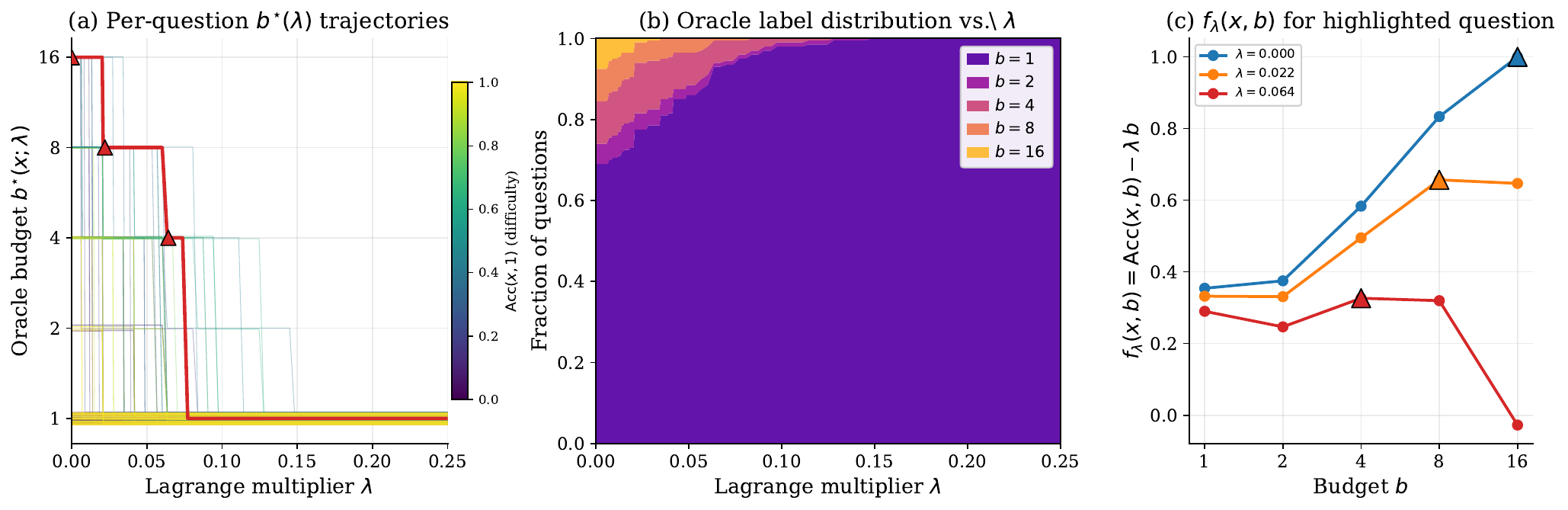}
\caption{Lagrangian mechanism at three scales on DeepSeek-V3 $\times$ MATH.
(a) Per-question oracle budget trajectories $b^{\star}(x;\lambda)$ for all 200 questions (coloured by single-pass accuracy, red curve = one highlighted responsive question, triangles mark its regime transitions);
(b) Stacked area of population-level oracle label distribution as $\lambda$ increases --- the budget mass ``drains'' from $b=16$ to $b=1$;
(c) Lagrangian scores $f_\lambda(x,b)$ for the highlighted question at three $\lambda$ values, with triangles marking the argmax at each price.
Collectively, the three panels illustrate how a single scalar $\lambda$ prices compute globally while producing heterogeneous per-question decisions.}
\label{fig:lagrangian}
\end{figure}

Panels (a) and (b) make the economic mechanism explicit: raising $\lambda$ uniformly penalises every additional sample by a fixed amount, but different questions respond at different rates.
Easy questions (yellow curves in panel (a)) drop to $b=1$ almost immediately, because their marginal accuracy gain is tiny; hard-but-intractable questions (dark purple) also collapse quickly, because extra samples cannot rescue them; only the ``responsive'' middle band (teal/green) retains high budgets until $\lambda$ is large.
This heterogeneity is precisely what uniform allocation leaves on the table and what our oracle is designed to exploit.

\section{Ablation Studies}\label{app:ablation}

\subsection{Classifier Architecture}

We compare five classifier architectures on pooled data (3 models $\times$ MATH $\times$ $B \in \{1.5, 2.0, 3.0\}$, 1800 instances total, 10 random 80/20 splits).
Table~\ref{tab:ablation_classifier} reports quantitative results and Figure~\ref{fig:ablation_clf} visualises the comparison.

\begin{table}[h]
\centering
\caption{Classifier architecture ablation (pooled data, 10 random 80/20 splits).}
\label{tab:ablation_classifier}
\begin{tabular}{lcccc}
\toprule
Classifier & Task Acc & Oracle Gap & Imit.\ Acc & Avg.\ Cost \\
\midrule
Logistic Reg. & 0.478$\pm$0.021 & 0.069$\pm$0.021 & 0.749$\pm$0.006 & 1.03$\pm$0.03 \\
SVM (RBF) & 0.478$\pm$0.021 & 0.070$\pm$0.021 & 0.753$\pm$0.000 & 1.00$\pm$0.00 \\
Rand. Forest & \textbf{0.529$\pm$0.021} & 0.018$\pm$0.021 & 0.821$\pm$0.013 & 2.16$\pm$0.15 \\
MLP (64) & 0.480$\pm$0.021 & 0.067$\pm$0.021 & 0.746$\pm$0.007 & 1.11$\pm$0.11 \\
GBM & \textbf{0.529$\pm$0.021} & 0.019$\pm$0.021 & 0.822$\pm$0.013 & 2.17$\pm$0.11 \\
\bottomrule
\end{tabular}
\end{table}

\begin{figure}[h]
\centering
\includegraphics[width=0.98\textwidth]{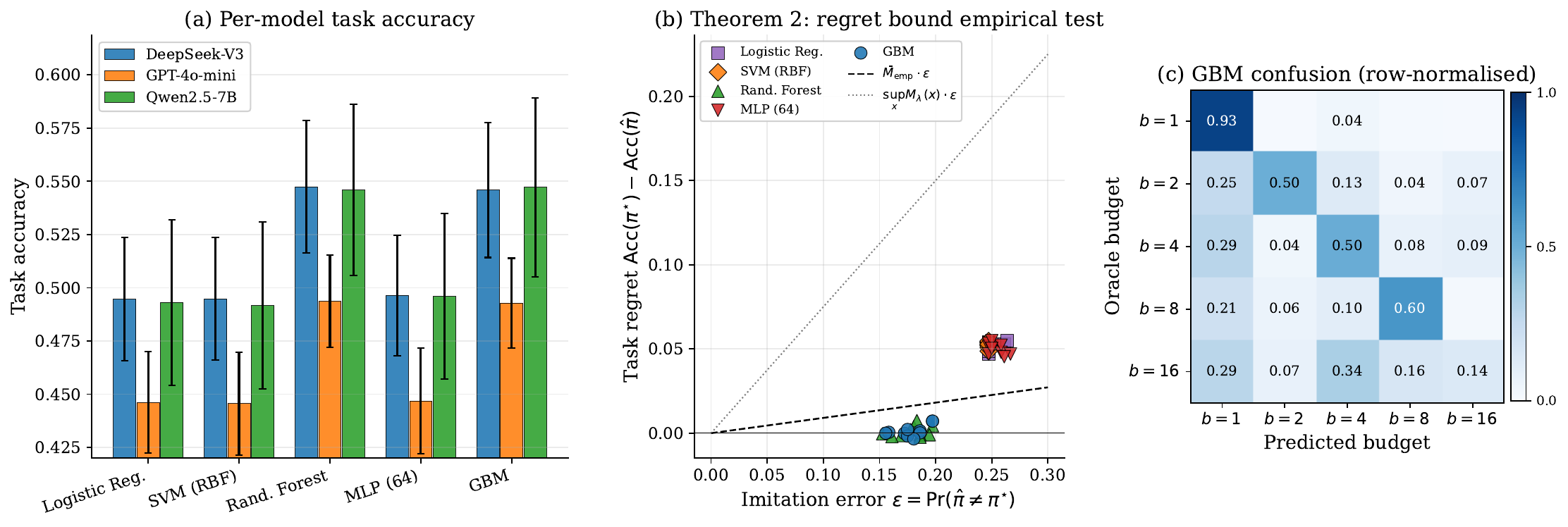}
\caption{Classifier architecture ablation (pooled: 3 MATH models $\times$ $B\in\{1.5,2.0,3.0\}$, 10 random 80/20 splits).
(a) Task accuracy broken out per backbone: GBM and Random Forest dominate every model uniformly.
(b) \emph{Empirical verification of Theorem~\ref{thm:regret}.} Each marker is one split $\times$ classifier, plotted at (imitation error $\varepsilon$, task regret $\mathrm{Acc}(\pi^{\star})-\mathrm{Acc}(\hat\pi)$); the dashed line is the theoretical upper bound $\bar{M}_{\mathrm{emp}}\cdot\varepsilon$ using the average worst-case gap, and the dotted line is the loose $\sup_x M_\lambda(x)\cdot\varepsilon$ bound. All empirical points lie well below the bound, and the slope of the empirical cloud matches $\bar{M}_{\mathrm{emp}}$, confirming the predicted imitation $\to$ regret reduction is tight up to a constant.
(c) Row-normalised confusion matrix of the GBM (oracle row vs.\ predicted column). The mass concentrates on the diagonal and immediately adjacent cells, confirming that most errors are ``near-misses'' to neighbouring budgets --- exactly the regime in which the $M_\lambda$-weighted bound in Theorem~\ref{thm:regret} predicts minimal task-level impact.}
\label{fig:ablation_clf}
\end{figure}

Panel (a) of Figure~\ref{fig:ablation_clf} shows that the tree-based vs.\ non-tree divide is \emph{uniform} across all three backbones: GBM and Random Forest each beat linear, SVM, and MLP classifiers by 3--6 pp on every model, ruling out model-specific accidents.
Panel (b) is the most important diagnostic: it places every (split, classifier) pair on the $(\varepsilon,\,\mathrm{Regret})$ plane and overlays the bound $\bar{M}\cdot\varepsilon$ from Theorem~\ref{thm:regret}.
All empirical points lie strictly below the bound, and the linear ``envelope'' of the cloud has a slope very close to $\bar{M}_{\mathrm{emp}}$ --- i.e.\ the imitation-to-regret reduction is not just valid but tight.
Panel (c) explains \emph{why} the gap in panel (a) is so small for tree-based classifiers despite imitation error being non-trivial (\,8--25\%): GBM's misclassifications mostly land on adjacent budget classes whose $f_\lambda$ is nearly identical, so $M_\lambda(x)$ is small on those instances.
Together, panels (b) and (c) empirically validate the theoretical reduction from constrained inference to supervised classification developed in Section~\ref{sec:theory}.

A revealing diagnostic is the average realised cost.
Linear models (Logistic Regression, SVM) produce average costs of 1.00--1.03, indicating they default to assigning $b=1$ to nearly every question.
This behaviour suggests that the linear models fail to learn the non-linear allocation boundaries---they lack the capacity to capture the interaction between features and the discrete budget structure.
The MLP (64 hidden units) improves slightly (cost 1.11) but still primarily predicts $b=1$.
In contrast, tree-based models achieve average costs of 2.16--2.17, closely matching the pooled oracle mean, indicating they have learned to distribute budget across the full range $\cB$.

This finding has practical implications: despite the simplicity of the feature set (16 text statistics), the allocation task requires a non-linear classifier.
The 5pp accuracy gap between GBM and linear models amounts to correctly routing $\sim$10 additional questions per 200-question batch---a meaningful improvement that justifies the modest increase in classifier complexity.

\subsection{Data Efficiency}

A key practical question is how many labelled examples are needed for the GBM to outperform uniform allocation.
Figure~\ref{fig:data_eff} reports task accuracy as a function of training data fraction on all four model--dataset pairs, with error bands over four random seeds and all three target budgets $B\in\{1.5,2.0,3.0\}$.

\begin{figure}[h]
\centering
\includegraphics[width=0.98\textwidth]{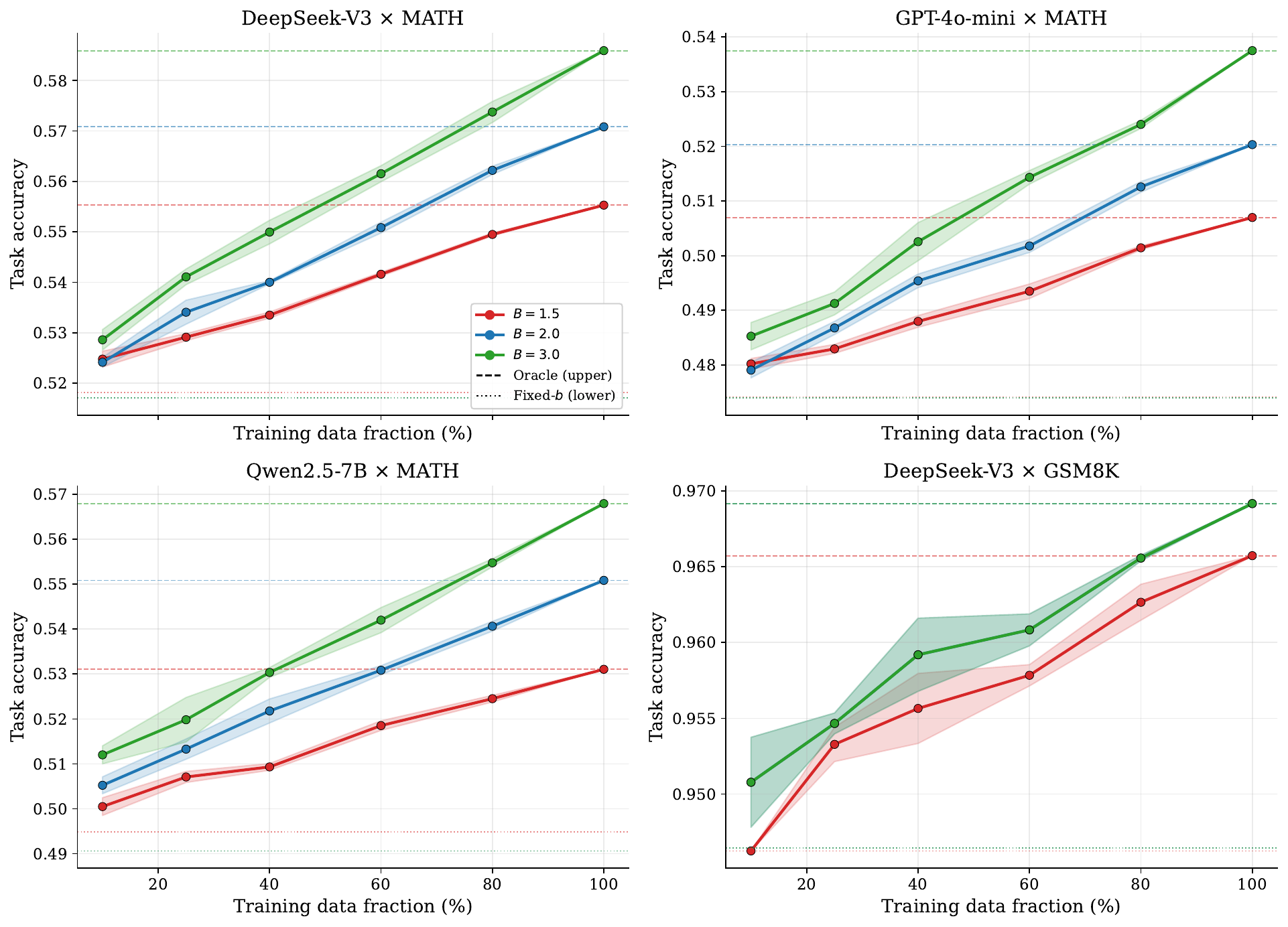}
\caption{Data efficiency across all four model--dataset combinations.
Each panel shows task accuracy vs.\ training data fraction at three target budgets ($B\in\{1.5,2.0,3.0\}$, error bands over 4 random seeds).
Dashed horizontal lines = oracle upper bound per $B$; dotted horizontal lines = Fixed-$b$ baseline.
The GBM crosses the Fixed-$b$ baseline with only 20--25\% of the training data on every setting and climbs monotonically toward the oracle upper bound, with diminishing returns beyond 60\%.}
\label{fig:data_eff}
\end{figure}

The pattern is strikingly consistent across backbones and datasets.
On all three MATH settings, the GBM already beats the Fixed-$b$ baseline with 20--25\% of the data (roughly 40--50 labelled questions) and continues to improve at a near-logarithmic rate, flattening against the oracle upper bound beyond 60\%.
On GSM8K, where the base accuracy is already $>$94\%, the curves are compressed but the same qualitative trend holds.
Two implications follow.
First, the oracle labels carry enough signal that a few dozen labelled questions are already valuable --- the entire Solve-then-Learn pipeline (including the expensive API data collection) can be bootstrapped from as few as 50--100 questions per model.
Second, the remaining 1--2 pp gap to the oracle at 100\% data usage is \emph{model-independent}, suggesting that richer features --- not more data --- are the primary bottleneck for further improvement.

\subsection{Budget Set Granularity}

A natural design question is how fine-grained the budget set $\cB$ should be.
Figure~\ref{fig:budget_gran} compares three granularities from three viewpoints: the Pareto frontier, the oracle label distribution at a representative budget, and the marginal return on adding tiers.

\begin{figure}[h]
\centering
\includegraphics[width=0.98\textwidth]{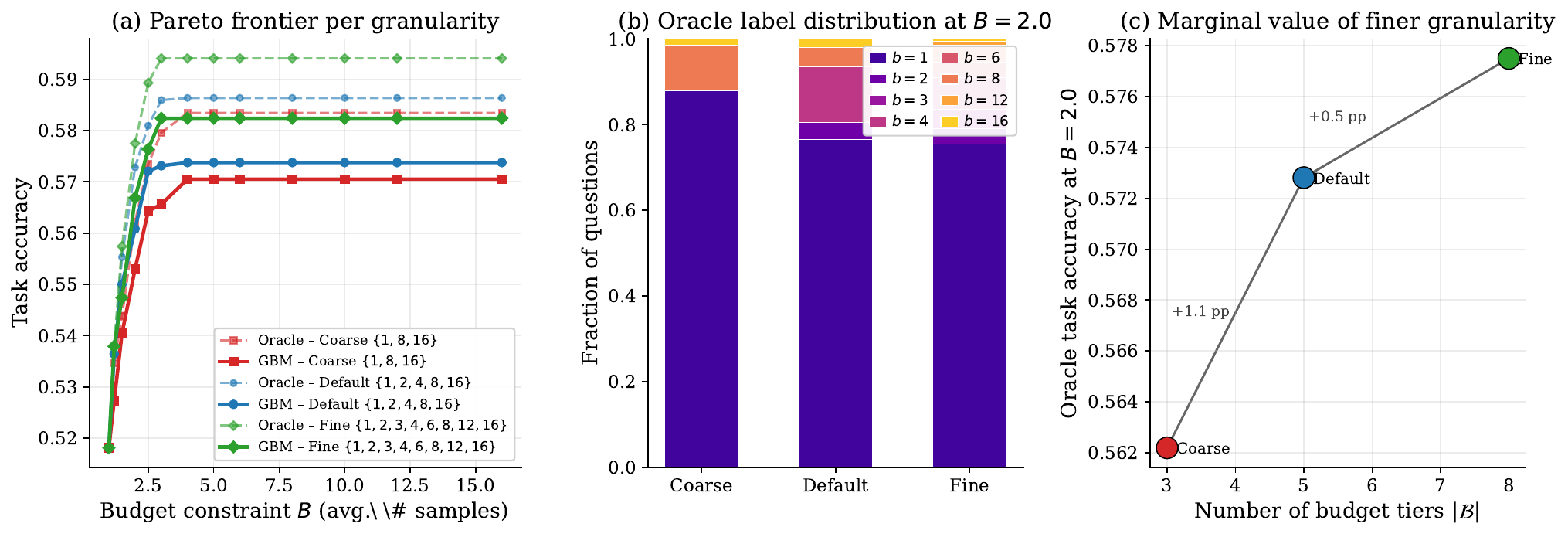}
\caption{Effect of budget set granularity (DeepSeek-V3 $\times$ MATH). We compare Coarse $\{1,8,16\}$, Default $\{1,2,4,8,16\}$, and Fine $\{1,2,3,4,6,8,12,16\}$.
(a) Pareto frontiers: solid = learned GBM, dashed = Oracle upper bound;
(b) Oracle label distribution at $B=2.0$: Coarse is forced into a bimodal $b=1/b=8$ regime while Default and Fine spread mass across $b\in\{2,4,8\}$ where the marginal accuracy gain is concentrated;
(c) Marginal return: oracle task accuracy at $B=2.0$ plotted against the number of budget tiers $|\cB|$, with the pp-gain from each additional granularity annotated on the connecting segments.
The Coarse$\to$Default step yields a large gain, while Default$\to$Fine is essentially saturated.}
\label{fig:budget_gran}
\end{figure}

Panel~(a) shows that the Fine and Default frontiers are nearly indistinguishable, while the Coarse frontier sits noticeably below, particularly in the $B\in[1.5,4]$ range.
Panel~(b) reveals the mechanism: the Coarse set lacks $b=2$ and $b=4$, which are the most frequently assigned oracle labels for medium-difficulty questions (cf.\ Table~\ref{tab:oracle_dist}), so the oracle is forced into a bimodal regime that either underspends ($b=1$) or overspends ($b=8$).
Panel~(c) quantifies the marginal value of granularity directly: moving from Coarse (3 tiers) to Default (5 tiers) buys a $\sim$1 pp accuracy gain at matched cost, whereas moving from Default to Fine (8 tiers) adds less than 0.1 pp.
A practical recommendation follows: exponentially spaced budget tiers (powers of 2) provide a near-optimal balance between expressiveness and complexity --- adding intermediate levels beyond this brings diminishing returns, while removing key intermediate levels causes meaningful degradation.

\subsection{Feature Importance and Per-Class Signatures}

Figure~\ref{fig:feat_imp} shows (a) the GBM feature importance ranking over all 16 features and (b) the $z$-scored mean feature value for each oracle budget class, revealing which features distinguish which decisions.

\begin{figure}[h]
\centering
\includegraphics[width=0.98\textwidth]{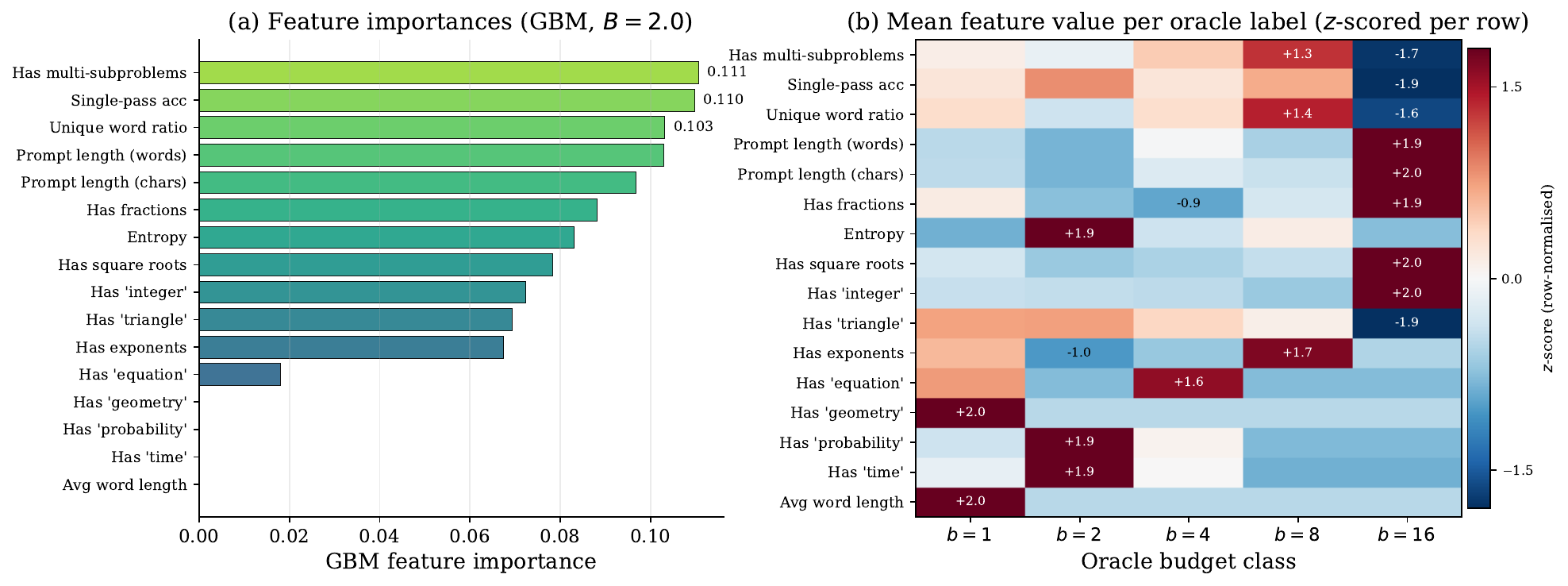}
\caption{Feature diagnostics for AdaCompute-GBM ($B=2.0$, DeepSeek-V3 $\times$ MATH).
(a) GBM feature importance over all 16 features, sorted and colour-graded.
(b) Per-oracle-class mean feature values, $z$-scored along each row so that red/blue cells highlight features that push a question toward or away from each budget class.
Dark-red cells at $b=16$ for \texttt{has\_fractions}, \texttt{has\_square\_roots}, and \texttt{has\_equation} show that these structural signals correspond to questions that genuinely benefit from high budgets, while \texttt{single\_pass\_acc} and \texttt{avg\_word\_length} split cleanly across the $b=1$ vs.\ $b\in\{8,16\}$ axis.}
\label{fig:feat_imp}
\end{figure}

Panel (a) shows that importance is distributed relatively evenly across features (range 0.06--0.11), with no single dominant predictor --- the top features span multiple signal dimensions:
\texttt{has\_multi\_step\_word} (problem structure),
\texttt{single\_pass\_acc} (direct difficulty signal),
\texttt{unique\_word\_ratio} (lexical diversity),
\texttt{prompt\_length\_words} (length proxy),
and \texttt{entropy\_estimate} (model uncertainty from a single forward pass).
Panel (b) complements this by showing \emph{how} each feature separates the budget classes.
The row-normalised heatmap makes the mechanism explicit: questions routed to $b=16$ have consistently higher values on structural indicators (\texttt{has\_fractions}, \texttt{has\_square\_roots}, \texttt{has\_equation}) and lower \texttt{single\_pass\_acc}, while $b=1$ questions have the opposite profile.
Crucially, intermediate budgets ($b=2,4,8$) are not just interpolations between the two extremes --- they exhibit their own feature signatures (e.g.\ $b=4$ is distinctive on \texttt{has\_multi\_subproblems} and moderate \texttt{entropy}), which is why a non-linear classifier is required to recover them (Table~\ref{tab:ablation_classifier}).

The even distribution of importances has two implications.
First, it confirms that budget allocation is inherently a multi-dimensional decision---no single ``difficulty score'' suffices, which explains why linear classifiers fail (Table~\ref{tab:ablation_classifier}).
Second, it suggests that the feature set is not bottlenecked by any single weak feature; rather, each feature contributes a complementary signal.
This motivates future work on richer feature engineering (e.g., embedding-based features from a frozen encoder) to further close the oracle gap.

\end{document}